\documentclass{article}
\usepackage{amsmath,amssymb,amsthm}
\usepackage{mathtools}
\usepackage{textcomp} %

\usepackage{algorithm}
\usepackage[noend]{algpseudocode}

\usepackage{array}
\usepackage{booktabs}
\usepackage{multirow}
\usepackage{pbox}

\usepackage{hyperref}
\usepackage{float}
\usepackage{color} %
\newcommand{\Reals}{\mathbb{R}}

\newcommand{\abs}[1]{{\left|#1\right|}}
\newcommand{\set}[1]{\left\{#1\right\}}

\newcommand{\tuple}[1]{\left(#1\right)}
\newcommand{\range}[2]{\left[#1,\ #2\right]}
\newcommand{\lorange}[2]{\left(#1,\ #2\right]}

\newcommand{\fnc}[2]{#1\left(#2\right)}
\newcommand{\fncapp}[2]{\fnc{#1}{#2}}
\newcommand{\fncdef}[3]{#1:\ #2 \rightarrow #3}
\newcommand{\multiplicity}[1]{\mathbf{1}_{#1}}

\newcommand{\bigoh}[1]{\fncapp{\mathcal{O}}{#1}}

\newcommand{\wg}{\algname{WeightGen}}
\newcommand{\ProbSymbol}{P}
\newcommand{\Prob}[2][]{\fnc{\ProbSymbol_{#1}}{#2}}
\newcommand{\weight}{w}
\newcommand{\w}[1]{\fncapp{\weight}{#1}}
\newcommand{\wprob}[1]{\fncapp{P_{\weight}}{#1}}
\newcommand{\tilt}{r}
\newcommand{\bound}[1]{\hat{#1}}
\newcommand{\formula}{\mathbf{F}}
\newcommand{\assignment}{F}

\newcommand{\DB}{\mathcal{D}}

\newcommand{\freq}[2][]{freq^{#1}\left(#2\right)}
\newcommand{\ItemCnt}{M}
\newcommand{\DBSize}{N}

\newcommand{\minsup}{\theta}
\newcommand{\minlen}{\lambda}

\newcommand{\Patternlang}{\mathcal{L}}
\newcommand{\Items}{\mathcal{I}}
\newcommand{\Transactions}{\mathcal{T}}
\newcommand{\Item}{i}
\newcommand{\Transaction}{t}
\newcommand{\Pattern}{p}

\newcommand{\ItemVar}{I}
\newcommand{\TransVar}{T}

\newcommand{\TransVars}{\mathbf{\TransVars}}

\newcommand{\Qual}{\varphi}
\newcommand{\Constraints}{\mathcal{C}}

\newcommand{\cpfim}{\algname{cp4im}}
\newcommand{\eclat}{\algname{Eclat}}
\newcommand{\lcm}{\algname{lcm}}

\newcommand{\algname}[1]{\textsc{#1}}

\newcommand{\parhead}[1]{\smallskip

\noindent \emph{#1}}

\newcommand{\mulr}[2]{\multirow{#2}{*}{#1}}
\newcommand{\twor}[1]{\mulr{#1}{2}}
\newcommand{\cencol}[2]{\multicolumn{#2}{c}{#1}}

\newcommand{\algName}{\textsc{Flexics}}
\newcommand{\algNameGen}{\textsc{GFlexics}}
\newcommand{\algNameSpec}{\textsc{EFlexics}}

\newtheorem{theorem}{Theorem}

\begin{document}
\title{Flexible~constrained~sampling with~guarantees for~pattern~mining}

\author{%
	Vladimir Dzyuba\textsuperscript{1},
	Matthijs {van~Leeuwen}\textsuperscript{2},
	Luc De Raedt\textsuperscript{1}%
}

\date{}

\maketitle

{\let\thefootnote\relax\footnotetext{\textsuperscript{1} DTAI, KU Leuven, Belgium, \href{mailto:vladimir.dzyuba@cs.kuleuven.be,luc.deraedt@cs.kuleuven.be}{\texttt{firstname.lastname@cs.kuleuven.be}}}}

{\let\thefootnote\relax\footnotetext{\textsuperscript{2} LIACS, Leiden University, The Netherlands, \href{mailto:m.van.leeuwen@liacs.leidenuniv.nl}{\texttt{m.van.leeuwen@liacs.leidenuniv.nl}}}}
\section*{Abstract}

Pattern sampling has been proposed as a potential solution to the infamous pattern explosion. Instead of enumerating all patterns that satisfy the constraints, individual patterns are sampled proportional to a given quality measure. Several sampling algorithms have been proposed, but each of them has its limitations when it comes to 1) flexibility in terms of quality measures and constraints that can be used, and/or 2) guarantees with respect to sampling accuracy.

We therefore present \algName{}, the first flexible pattern sampler that supports a broad class of quality measures and constraints, while providing strong guarantees regarding sampling accuracy. To achieve this, we
leverage the perspective on pattern mining as a constraint satisfaction problem and build upon the latest advances in sampling solutions in SAT as well as existing pattern mining algorithms.
Furthermore, the proposed algorithm is applicable to a variety of pattern languages, which allows us to introduce and tackle the novel task of sampling sets of patterns.

We introduce and empirically evaluate two variants of \algName{}: 1) a generic variant that addresses the well-known itemset sampling task and the novel pattern set sampling task as well as a wide range of expressive constraints within these tasks, and 2) a specialized variant that exploits existing frequent itemset techniques to achieve substantial speed-ups. Experiments show that \algName{} is both accurate and efficient, making it a useful tool for pattern-based data exploration. 

\section{Introduction}%
\label{sec:introduction}

Pattern mining \cite{Agrawal1996} is an important and well-studied task in data mining. Informally, a pattern is a statement in a formal language that concisely describes a subset of a given dataset. Pattern mining techniques aim at providing comprehensible descriptions of coherent regions in the data. Many variations of pattern mining have been proposed in the literature, together with even more algorithms to efficiently mine the corresponding patterns. Best known is frequent pattern mining \cite{Aggarwal2014}, which includes frequent itemset mining and its extensions.

Traditional pattern mining methods enumerate all frequent patterns, though it is well-known that this usually results in humongous amounts of patterns (the infamous \emph{pattern explosion}). To make pattern mining more useful for exploratory purposes, different solutions to this problem have been proposed. Each of these solutions has its own advantages and disadvantages. \emph{Condensed representations} \cite{Calders2006} can often be efficiently mined, but generally still result in large numbers of patterns.  \emph{Top-$k$ mining} \cite{Zimmermann2014} is efficient but results in strongly related, redundant patterns showing a lack of diversity. \emph{Constrained mining} \cite{Nijssen2014} may result in too few or too many patterns, depending on the user-chosen constraints. \emph{Pattern set mining} \cite{Bringmann2010} takes into account the relationships between the patterns, which can result in small solution sets, but is computationally intensive.

In this paper, we study \emph{pattern sampling}, another approach that has been proposed recently: instead of enumerating all patterns, patterns are sampled one by one, according to a probability distribution that is proportional to a given quality measure. 
The promised benefits include: 1) flexibility in that potentially a broad range of quality measures and constraints can be used; 2) `anytime' data exploration, where a growing representative set of patterns can be generated and inspected at any time; 3) diversity in that the generated sets of patterns are independently sampled from different regions in the solution space. To be reliable, pattern samplers should provide theoretical guarantees regarding the sampling accuracy, i.e.,
the difference between the empirical probability of sampling a pattern and the (generally unknown) target probability determined by its quality. 
These properties are essential for pattern mining applications ranging from showing patterns directly to the user, where flexibility and the anytime property enable experimenting with and fine-tuning mining task formulations, to candidate generation for building pattern-based models, for which the approximation guarantees can be derived from those of the sampler.

While a number of pattern sampling approaches have been developed over the past years, they
are either inflexible (as they only support a limited number of quality measures and constraints), or do not provide theoretical guarantees concerning the sampling accuracy. 
At the algorithmic level, they follow standard sampling approaches such as Markov Chain Monte Carlo random walks over the pattern lattice \cite{Boley2009,Hasan2009,Boley2010}, or a special purpose sampling procedure tailored for a restricted set of itemset mining tasks \cite{Boley2011,Boley2012}.  
Although MCMC approaches are in principle applicable to a broad range of tasks, they often converge only slowly to the desired target distribution
and require the selection of the ``right'' proposal distributions. 

To the best of our knowledge, none of the existing approaches to pattern sampling takes advantage of the latest developments in sampling technology from the SAT-solving community, where a number of powerful samplers based on random hash functions and XOR-sampling have been developed \cite{Gomes2006,Chakraborty2013b,Ermon2013,Meel2016}. \wg{} \cite{Chakraborty2014}, one of the recent approaches,
possesses the benefits mentioned above: it is an anytime algorithm, it is flexible as it works with any distribution, 
it generates diverse solutions, and provides strong performance guarantees under reasonable assumptions.

In this paper, we show that the latest developments in sampling solutions in SAT are also relevant to pattern sampling
and essentially offer the same advantages. Our results build upon the view of pattern mining as constraint satisfaction, 
which is now commonly accepted in the data mining community  \cite{Guns2011}. 

\begin{table*}[t]

\newcommand{\headrow}[1]{{\scriptsize #1}}
\newcommand{\methodrow}[9]{#1 & #2 & \twor{#3} & \twor{#4} & #5 & \twor{#6} \\ #7 & #8 & & & #9 & \\[0.1cm]}

\caption{
Our method is
the first pattern sampler that
combines flexibility
with respect to
the choice of constraints and
sampling distributions with
strong theoretical guarantees.
}

\centering
\setlength{\tabcolsep}{5pt}

\begin{tabular}{lccccc}\toprule
\cencol{\twor{\headrow{Sampler}}}{1} & \headrow{Arbitrary}   & \headrow{Arbitrary}     & \headrow{Strong}     & \twor{\headrow{Efficiency}} & \headrow{Pattern set} \\
                                     & \headrow{constraints} & \headrow{distributions} & \headrow{guarantees} &                             & \headrow{sampling} \\
\midrule
\methodrow{\twor{ACFI \cite{Boley2009}}}{Minimal}{-}{-}{\twor{\checkmark}}{-}{}{frequency}{}
\methodrow{\twor{LRW \cite{Hasan2009}}}{\twor{\checkmark}}{\checkmark}{-}{{\scriptsize Implementation-}}{-}{}{}{{\scriptsize specific}}
\methodrow{\twor{FCA \cite{Boley2010}}}{Anti-/}{\checkmark}{-}{\twor{\checkmark}}{-}{}{monotonic}{}
\methodrow{TS (Two-step)}{\twor{-}}{-}{\checkmark}{\twor{\checkmark}}{-}{\cite{Boley2011,Boley2012}}{}{}
\midrule
\methodrow{\algName{}}{\twor{\algNameGen{}}}{\checkmark}{\checkmark}{\twor{\algNameSpec{}}}{\checkmark}{\small This paper}{}{}
\bottomrule
\end{tabular}

\label{table:comparison}

\end{table*} 
\parhead{Approach and contributions}
More specifically, 
we introduce \algName{}: a flexible pattern sampler that samples from distributions induced by a variety of pattern quality measures and allows for a broad range of constraints while still providing strong theoretical guarantees. Notably, \algName{} is, in principle, agnostic of the quality measure, as the sampler treats it as a black box. (However, its properties affect the efficiency of the algorithm.) The other building block is a \emph{constraint oracle} that enumerates all patterns that satisfy the constraints, i.e., a mining algorithm. The proposed approach allows converting an existing pattern mining algorithm into a sampler with guarantees. Thus, its flexibility is not limited by the choice of constraints and quality measures, but even allows tackling richer pattern languages, which we demonstrate by tackling the novel task of \emph{sampling sets of patterns}. \autoref{table:comparison} compares the proposed approach to alternative samplers; see Section~\ref{sec:relwork} for a more detailed discussion.

The main technical contribution of this paper consists of two variants of the \algName{} sampler, which are based on different constraint oracles. First, we introduce a generic variant, dubbed \algNameGen{}, that supports a wide range of pattern constraints, such as syntactic or redundancy-eliminating constraints. \algNameGen{} uses \cpfim{} \cite{Guns2011}, a declarative constraint programming-based mining system, as its oracle. Any constraint supported by \cpfim{} can be used without interfering with the umbrella procedure that performs the actual sampling task. Unlike the original version of \wg{} that is geared towards SAT, \algNameGen{} can handle cardinality constraints that are ubiquitous in pattern mining. Furthermore, we identify (based on previous research) the properties of the constraint satisfaction-based formalization of pattern mining that further improve the efficiency of the sampling procedure without affecting its guarantees and thus make it applicable to practical problems. We use \algNameGen{} to tackle a wide range of well-known itemset sampling tasks as well as the novel pattern set sampling task. Second, as it is well-known that generic solvers impose an overhead on runtime, we introduce a variant specialized towards frequent itemsets, dubbed \algNameSpec{}, which has an extended version of \eclat{} \cite{Zaki1997} at its core as oracle.

Experiments show that \algName{}' sampling accuracy is impressively high: in a variety of settings supported by the sampler, empirical frequencies are within a small factor of the target distribution induced by various quality measures. Furthermore, practical accuracy is substantially higher than theory guarantees. \algNameSpec{} is shown to be faster than its generic cousin, demonstrating that developing specialized solvers for specific tasks is beneficial when runtime is an issue. Finally, the flexibility of the sampler allows us to use the same approach to successfully tackle the novel problem of sampling pattern sets.
This demonstrates that \algName{} is a useful tool for pattern-based data exploration.

This paper is organized as follows.
We formally define the problem of pattern sampling in Section~\ref{sec:problem}. After reviewing related research in Section~\ref{sec:relwork}, we present the two key ingredients of the proposed approach in Section~\ref{sec:prelim}: 1) the perspective on pattern mining as a constraint satisfaction problem and 2) hashing-based sampling with \wg{}. In Section~\ref{sec:algorithm}, we present \algName{}, a flexible pattern sampler with guarantees. In particular, we outline the modifications required to adapt \wg{} to pattern sampling and describe the procedure to convert two existing mining algorithms into oracles suitable for use with \wg{}, which yields two
variants of \algName{}. In Section~\ref{sec:set-sampling}, we introduce the pattern set sampling task and describe how it can be tackled with \algName{}. We also outline sampling non-overlapping tilings, an example of pattern set sampling that is studied in the experiments. The experimental evaluation in Section~\ref{sec:experiments} investigates the accuracy, scalability, and flexibility of the proposed sampler. We discuss its potential applications, advantages, and limitations in Section~\ref{sec:discuss}. Finally, we present our conclusions in Section~\ref{sec:conclusion}.

\section{Problem definition}%
\label{sec:problem}

Here we present
a high-level definition of
the task that we consider
in this paper;
for concrete instances and examples, 
see Sections \ref{sec:prelim}
and \ref{sec:set-sampling}.
The pattern sampling problem
is formally defined as follows:
given a dataset $\DB$,
a pattern language $\Patternlang$,
a set of constraints $\Constraints$, and
a quality measure
$\fncdef{\Qual}{\Patternlang}{\Reals^+}$,
generate random patterns
that satisfy constraints 
in $\Constraints$
with probability proportional 
to their qualities:
\begin{align*}
\Prob[\Qual]{\Pattern} = \begin{dcases}
\fnc{\Qual}{\Pattern} / Z_{\Qual} & \text{if }\Pattern\in\Patternlang\text{ satisfies }\Constraints \\
0 & \text{otherwise} 
\end{dcases}
\end{align*}
where $Z_{\Qual}$ is
an (often unknown) 
normalization constant.

A quality measure quantifies
the domain-specific
interestingness of a pattern.
The choice of
a quality measure and constraints
allows a user to express 
her analysis requirements.
The sampling procedure 
meets these requirements by
satisfying the constraints and
generating high-quality patterns
more frequently.
Thus, sampled patterns are
a representative subset of 
all interesting regularities
in the dataset.

Pattern set mining is 
an extension of pattern mining, 
which considers sets of patterns 
rather than individual patterns.
Despite its popularity,
we are not aware of the existence of 
pattern set samplers. 
The task of pattern set sampling 
can easily be formalized 
as an extension of pattern sampling,
where we sample sets of patterns 
$s \subset \Patternlang$, 
and the constraints $\Constraints$ 
as well as the quality measure $\Qual$ 
are specified over sets of patterns 
(from $2^\Patternlang$)
rather than individual patterns 
(from $\Patternlang$).
\section{Related work} \label{sec:relwork}

We here focus on 
two classes of
related work, i.e.,
1) pattern mining as 
constraint satisfaction
and 2) pattern sampling.

\parhead{Constrained pattern mining}
The study of constraints
has been a prominent subfield
of pattern mining.
A wide range of constraint classes
were investigated, including
anti-monotonic constraints
\cite{Agrawal1996},
convertible constraints
\cite{Pei2000}, and others.
Another development 
of these ideas led to
the introduction of
global constraints that
concern multiple patterns and
to the emergence of 
\emph{pattern set mining}
\cite{Knobbe2006,DeRaedt2007}.
Furtheremore, generic mining systems
that could freely combine
various constraints were proposed
\cite{Bucila2003,Bonchi2009}.

These insights allowed to draw
a connection between 
pattern mining and 
constraint satisfaction in AI,
e.g., SAT or constraint programming (CP).
As a result, declarative mining systems,
which use generic constraint solvers
to mine patterns according
to a declarative specification
of the mining task, were proposed.
For example, CP was used to develop 
first declarative systems
for itemset mining \cite{Guns2011}
and pattern set mining
\cite{Khiari2010,Guns2013}.
Recently, declarative approaches 
have been extended to support
sequence mining \cite{Kemmar2014} and
graph mining \cite{Paramonov2015}.

Constraint-based systems allow 
a user to specify a wide range of 
pattern constraints and thus
provide tools to alleviate
the pattern explosion.
However, the underlying solvers 
use systematic search,
which affects the order
of pattern generation and thus 
prevents them from being used
in a truly anytime manner
due to low diversity of
consecutive solutions.
Similarly, pattern set miners 
that directly aim at obtaining
diverse result sets
typically incur prohibitive
computational costs as the size of 
the pattern space grows.

\parhead{Pattern sampling}
In this paper we focus 
on the approaches that 
directly aim at generating
random pattern collections
rather than the methods
whose goal is to estimate 
dataset or pattern language
statistics; cf. Shervashidze et al.
\cite{Shervashidze2009}.

\autoref{table:comparison}
compares our method with
the approaches described
in Section \ref{sec:introduction},
namely MCMC and \emph{two-step} samplers
\cite{Boley2011,Boley2012}.
We further break down MCMC samplers
into three groups: ACFI,
the very first uniform sampler
developed for approximate
counting of frequent itemsets
\cite{Boley2009};
LRW, a generic approach
based on random walks over
pattern lattice 
\cite{Hasan2009}; and 
FCA, a sampler, which uses
Markov chains
based on insights from 
formal concept analysis 
\cite{Boley2010}.

Although MCMC samplers
provide theoretical guarantees,
in practice, their convergence is
often slow and hard to diagnose.
Solutions such as long burn-in or
heuristic adaptations either
increase the runtime or
weaken the guarantees.
Furthermore, ACFI is tailored
for a single task;
FCA only supports 
anti-/monotone constraints; and
LRW checks constraints locally,
while building the neighborhood 
of a state, which might require 
advanced reasoning and
extensive caching.
Two-step samplers, while 
provably accurate and efficient,
only support 
a limited number of 
weight functions and
do not support constraints.
\section{Preliminaries} \label{sec:prelim}

We first outline 
itemset mining,
a prototypical 
pattern mining task, 
and formalize it as 
a CSP and then
describe \wg{}, a 
hashing-based 
sampling algorithm.

\subsection{Itemset mining}
\label{sec:prelim:cp4im}

Itemset mining is 
an instance of pattern mining
specialized for binary data.
Let $\Items=\set{1 \ldots \ItemCnt}$ 
denote a set of items.
A dataset $\DB$ is a bag of 
transactions over $\Items$,
where each transaction $\Transaction$ is 
a subset of $\Items$, i.e.,
$\Transaction \subseteq \Items$;
$\Transactions=\set{1 \ldots \DBSize}$ is
a set of transaction indices.
The pattern language $\Patternlang$ 
also consists of sets of items, 
i.e., $\Patternlang = 2^{\Items}$.
An itemset $\Pattern$ occurs in 
a transaction $\Transaction$, 
iff $\Pattern \subseteq \Transaction$.
The frequency of $\Pattern$ is
the number of transactions
in which it occurs:
$\freq{\Pattern} = \abs{%
	\set{\Transaction \in \DB \mid%
	\Pattern \subseteq \Transaction}%
}$.
In labeled datasets, 
a transaction
has a label from $\set{-,+}$;
$freq^{-,+}$ are defined accordingly.

We first give 
a brief overview of
the general approach to 
solving CSPs and then
present a formalization of
itemset mining as a CSP,
following that of 
\cpfim{} \cite{Guns2011}.
Formally, a CSP is comprised of
\emph{variables} along with
their \emph{domains} and 
\emph{constraints} over 
these variables.
The goal is to find
a solution, i.e.,
an assignment of values
to all variables that
satisfies all constraints.
Every constraint is implemented 
by a \emph{propagator}, i.e.,
an algorithm that
takes domains as input and
removes values that do not
satisfy the constraint.
Propagators are activated
when variable domains change, e.g.,
by the search mechanism 
or other propagators.
A CSP solver is typically
based on depth-first search.
After a variable is assigned a value, 
propagators are run
until domains cannot be 
reduced any further.
At this point, three cases are possible:
1) a variable has an empty domain, i.e.,
the current search branch has failed and
backtracking is necessary, 
2) there are unassigned variables, i.e.,
further branching is necessary, or
3) all variables are assigned a value,
i.e., a solution is found.

\begin{table}[t]

\caption{Constraint programming formulations of
common itemset mining constraints.
$\ItemVar_i = 1$ implies that
item $i$ is included in
the current (partial) solution, 
whereas $\TransVar_t = 1$
implies that it covers
transaction $t$.}

\vspace{2pt}

\newcommand{\constraintline}[9]{$#1$ & $#9$ & $\forall #2 \in #3$ & $#4$ & $#5$ & $\sum_{#6 \in #7}{#8}$}
\setlength{\tabcolsep}{10pt}
\centering \begin{tabular}{lll@{\hskip 4pt}l@{\hskip 4pt}c@{\hskip 4pt}l}\toprule
Constraint & Parameters & \multicolumn{4}{l}{CP formulation} \\
\midrule
\constraintline{coverage}              {t}{\Transactions}{\TransVar_t = 1}{\Leftrightarrow}{i}{\Items}       {\ItemVar_i\left(1 - \DB_{ti}\right) = 0}{} \\
\midrule
\constraintline{\fnc{minfreq}{\minsup}}{i}{\Items}       {\ItemVar_i = 1} {\Rightarrow}    {t}{\Transactions}{\TransVar_t\DB_{ti} \geq \minsup \times \abs{\DB}}{\minsup \in \lorange{0}{1}} \\
\constraintline{closed}                {i}{\Items}       {\ItemVar_i = 1} {\Leftrightarrow}{t}{\Transactions}{\TransVar_t\left(1 - \DB_{ti}\right) = 0}{} \\
\constraintline{\fnc{minlen}{\minlen}} {t}{\Transactions}{\TransVar_t = 1}{\Rightarrow}    {i}{\Items}       {\ItemVar_i\DB_{ti} \geq \minlen}{\minlen \in \range{1}{\ItemCnt}} \\
\bottomrule
\end{tabular}
\label{table:constraints}
\end{table}

Let $\ItemVar_i$ denote a variable 
corresponding to each item;
$\TransVar_t$ a variable 
corresponding to each transaction; and
${\DB}_{ti}$ a constant
that is equal to $1$, if
item $i$ occurs in transaction $t$,
and $0$ otherwise.
Variables $\ItemVar_i$ and 
$\TransVar_t$ are \emph{binary}, i.e.,
their domain is $\set{0,1}$.
Each CSP solution 
corresponds to a single itemset.
Thus, for example, 
$\ItemVar_i = 1$ implies that
item $i$ is included in
the current (partial) solution, 
whereas $\TransVar_t = 0$
implies that transaction $t$
is \emph{not} covered by it.
\autoref{table:constraints}
lists some of the 
most common constraints.
The $coverage$ constraint 
essentially models
a dataset query and ensures that
if the item variable assignment
corresponds to an itemset $\Pattern$,
only those transaction variables
that correspond to 
indices of transactions 
where $\Pattern$ occurs,
are assigned value $1$.
Other constraints allow users 
to remove uninteresting solutions,
e.g., redundant non-$closed$ itemsets.
Most solvers provide
facilities for enumerating
all solutions in sequence,
i.e., to enumerate all patterns.

In contrast to
hard constraints,
\emph{quality measures} 
are used to describe 
soft user preferences 
with respect to
interestingness of patterns.
Common quality measures 
concern frequency, 
e.g., $\Qual \equiv freq$,
discriminativity in 
a labeled dataset,
e.g., \emph{purity}
$\fnc{\Qual}{\Pattern}=\left. \max \set{%
	\freq[+]{\Pattern}, \freq[-]{\Pattern}%
} \middle/ \freq{\Pattern} \right.$,
etc.

\subsection{WeightGen}
\wg{} \cite{Chakraborty2014} is 
an algorithm for approximate 
weighted sampling of 
satisfying assignments (solutions)
of a Boolean formula
that only requires access to
an efficient constraint oracle that
enumerates the solutions, 
e.g., a SAT solver.
The core idea consists in
partitioning the solution space
into a number of ``cells'' and
sampling a solution from
a random cell.
Partitioning with desired properties
is obtained via augmenting
the original problem with
random XOR constraints.
Theoretical guarantees stem from 
the properties of
uniformly random 
XOR constraints.
The sequel follows
Sections 3-4 in 
Chakraborty et al.
\cite{Chakraborty2014}.

\parhead{Problem statement and guarantees}
Formally, let $\formula$ 
denote a Boolean formula;
$\assignment$ a satisfying
variable assignment of $\formula$;
$\ItemCnt$ the total number of variables;
$\w{\cdot}$ a black-box weight function
that for each $\assignment$
returns a number in $\lorange{0}{1}$;
and $\weight_{min}$ (resp. $\weight_{max}$)
the minimal (resp. maximal) weight over
all satisfying assignments of $\formula$.
The weight function induces 
the probability distribution
over satisfying assignments of $\formula$,
where $\wprob{\assignment} = \left.%
	\w{\assignment} \middle/ \sum{\w{\assignment'}}%
\right.$.
Quantity $\tilt = \left.%
	\weight_{max} \middle/ \weight_{min}%
\right.$
is the (possibly unknown) 
\emph{tilt} of the distribution
$\ProbSymbol_{\weight}$.

Given a user-provided 
upper bound on tilt 
$\bound{\tilt} \geq \tilt$ and 
a desired sampling error tolerance
$\kappa \in \tuple{0,\ 1}$
(the lower $\kappa$, the tighter
the bounds on the sampling error), 
\algname{WeightGen} generates
a random solution $\assignment$.
Performance guarantees concern
both accuracy and efficiency of
the algorithm and
depend on the parameters and
the number of variables $\ItemCnt$;
see Section~\ref{sec:algorithm} for details.

\parhead{Algorithm}
Recall that the core idea
that underlies sampling 
with guarantees is
partitioning the overall solution space
into a number of random cells
by adding random XOR constraints.
\wg{} proceeds in two phases:
1) the estimation phase and 
2) the sampling phase.
The goal of the \emph{estimation phase} is
to estimate the number of
XOR constraints necessary
to obtain a ``small'' cell,
where the required cell weight
is determined by the desired 
sampling error tolerance.

The \emph{sampling phase} 
starts with applying
the estimated number 
of XOR constraints.
If it obtains a cell
whose total weight lies
within a certain range, 
which depends on $\kappa$,
a solution is sampled exactly
from all solutions in the cell;
otherwise, it adds a new
random XOR constraint.
However, the number of 
XOR constraints that 
can be added is limited.
If the algorithm cannot obtain
a suitable cell,
it indicates failure and
returns no sample.

Both phases make use of
a \emph{bounded} oracle that
terminates as soon as
the total weight of 
enumerated solutions 
exceeds a predefined number.
It enumerates solutions of
the original problem $\formula$
augmented with the XOR constraints.
An individual XOR constraint
over variables $\mathbf{X}$ has the form 
$\bigotimes{b_i \cdot X_i} = b_0$, 
where $b_{0 \mid i} \in \set{0,1}$.
The coefficients $b_i$ determine
the variables involved
in the constraint, whereas
the \emph{parity bit} $b_0$ determines
whether an even or an odd number of
variables must be set to $1$.
Together, $m$ XOR constraints 
identify one cell belonging
to a partitioning of
the overall solution space
into $2^m$ cells.

The core operation of \wg{}
involves drawing 
coefficients uniformly at random,
which induces a random 
partitioning of
the solution space 
that satisfies the 
\emph{$3$-wise independence property}, 
i.e., knowing the cells for
two arbitrary assignments
does not provide any information
about the cell for
a third assignment \cite{Gomes2006}.
This ensures desired
statistical properties of
random partitions, required for 
the theoretical guarantees.
The reader interested in 
further technical details 
should consult 
Appendix \ref{sec:appendix:wg}
and Chakraborty et al.
\cite{Chakraborty2014}.
\section{Flexics: Flexible pattern sampler with guarantees}
\label{sec:algorithm}

In this paper,
we propose \algName{},
a pattern sampler
that uses \wg{} as 
the umbrella sampling procedure.
To this end, we 1) extend it
to CSPs with binary variables,
a class of problems that is
more general than SAT and that
includes pattern mining
as described in Section \ref{sec:prelim};
2) augment existing 
pattern mining algorithms
for use with \wg{}; and
3) investigate the properties of
pattern quality measures
in the context of 
\wg{}'s requirements.

\wg{} was originally presented
as an algorithm to sample
solutions of the SAT problem.
Pattern mining problems
cannot be efficiently tackled by 
pure Boolean solvers 
due to the prominence of 
cardinality constraints 
(e.g., $minfreq$).
However, we observe that 
the core sampling procedure
is applicable to any CSP 
with binary variables,
as its solution space
can be partitioned with
XOR constraints 
in the required manner.

Based on this insight,
we present two variants of 
\algName{} that differ 
in their oracles.
Each oracle is essentially 
a pattern mining algorithm
extended to support 
XOR constraints along with
common constraints on patterns.
The first one, 
dubbed \algNameGen{}, 
builds upon the generic
formalization and 
solving techniques 
described in Section~%
\ref{sec:prelim} 
and thus supports 
a wide range of constraints.
Owing to the properties of
the $coverage$ constraint,
XOR constraints only need
to involve item variables\footnote{%
In other words, item variables $\ItemVar$
are the \emph{independent support}
of a pattern mining CSP.},
which makes them relatively short,
mitigating the computational overhead.
Moreover, this perspective helps us
design the second approach,
dubbed \algNameSpec{},
which uses an extension of
\eclat{} \cite{Zaki1997},
a well-known mining algorithm,
as an oracle.
It is tailored for a single task
(frequent itemset mining,
i.e., it only supports 
the $minfreq$ constraint),
but is capable of 
handling larger datasets.
We describe each oracle in detail 
in the following subsections.

Given a dataset $\DB$,
constraints $\Constraints$,
a quality measure $\Qual$,
and the error tolerance 
parameter $\kappa \in \tuple{0,\ 1}$,
\algName{} first constructs a CSP 
corresponding to the task of 
mining patterns satisfying 
$\Constraints$ from $\DB$.
It then determines parameters for
the sampling procedure, 
including the appropriate
number of XOR constraints, and
starts generating samples.
To this end, it uses one of 
the two proposed oracles
to enumerate patterns
that satisfy $\Constraints$
\emph{and} random XOR constraints.
Both variants of \algName{}
support sampling from
black-box distributions derived 
from quality measures and, 
most importantly, preserve 
the theoretical guarantees of \wg{}%
\footnote{\autoref{th:accuracy}
corresponds to and
follows from Theorem 3
of Chakraborty et al.
\cite{Chakraborty2014}.}:

\begin{theorem}
The probability 
that \algName{} samples
a random pattern $\Pattern$ that 
satisfies constraints $\Constraints$
from a dataset $\DB$,
lies within a bounded range
determined by the quality of
the pattern $\fnc{\Qual}{\Pattern}$ 
and $\kappa$:
\begin{align*}
\frac{\fnc{\Qual}{\Pattern}}{Z_{\Qual}} \times%
		\frac{1}{1 + \fnc{\varepsilon}{\kappa}} \leq%
	\Prob{\fnc{\text{\algName{}}}{\DB, \Constraints, \Qual; \kappa} = \Pattern} \leq
\frac{\fnc{\Qual}{\Pattern}}{Z_{\Qual}} \times%
		\left(1 + \fnc{\varepsilon}{\kappa}\right)
\end{align*}
\label{th:accuracy}
\end{theorem}

\begin{proof}
Theorem 3 of Chakraborty et al. \cite{Chakraborty2014} states:
\begin{align*}
\left.\wprob{\assignment} \middle/ \left(1 + \fnc{\varepsilon}{\kappa}\right)\right. \leq%
	\bound{\ProbSymbol}_{\assignment}%
\leq \wprob{\assignment} \times \left(1 + \fnc{\varepsilon}{\kappa}\right)
\end{align*}
\noindent where 
$\bound{\ProbSymbol}_{\assignment}$
denotes the probability that \wg{}
called with parameters 
$\bound{\tilt}$ and $\kappa$
samples the solution $\assignment$, 
$\wprob{\assignment} \propto \w{\assignment}$ 
denotes the target probability 
of $\assignment$, and
$\fnc{\varepsilon}{\kappa} = %
	\left(1+\kappa\right)
	\left(2.36+0.51/\left(1 - \kappa\right)^2\right) - 1$
denotes sampling error
derived from $\kappa$.

For technical purposes,
we introduce the notion of 
the \emph{weight} of a pattern as
its quality scaled to the range 
$\lorange{0}{1}$, i.e.,
$\fnc{\weight_{\Qual}}{\Pattern} = \fnc{\Qual}{\Pattern}/C$,
where $C$ is 
an arbitrary constant
such that $C \geq \max%
	\limits_{\Pattern \in\Patternlang}%
	\fnc{\Qual}{\Pattern}$.
The proof follows from
Theorem~3 of Chakraborty et al.
\cite{Chakraborty2014}
and the observation that
$\fnc{\text{\algName{}}}{\DB, \Constraints, \Qual; \kappa}$
is equivalent to
$\fnc{\text{\wg{}}}{\fnc{\text{CSP}}{\DB, \Constraints}, \weight_{\Qual}; \kappa}$.
The estimation phase
effectively corrects
for potential discrepancy
between $C$ and $Z_{\Qual}$.
\qed
\end{proof}

Furthermore, Theorem 4
of Chakraborty et al.
\cite{Chakraborty2014},
provides \emph{efficiency guarantees}:
the number of calls to 
the oracle is linear in 
$\bound{\tilt}$ and
polynomial in $\ItemCnt$ and 
$1 / \fnc{\varepsilon}{\kappa}$.
The assumption
that the tilt is
\emph{bounded from above}
by a reasonably low number is
the only assumption regarding 
a (black-box) weight function.
Moreover, it only affects 
the efficiency of
the algorithm, but
not its accuracy.

Thus, using a quality measure
with \algName{} requires
knowledge of two properties:
scaling constant $C$ and
tilt bound $\bound{\tilt}$.
In practice, both are
fairly easy to come up with for
a variety of measures.
For example, for $freq$ and $purity$,
$C = \abs{\DB}$,
$\bound{\tilt} = \minsup^{-1}$ and
$C = 1$, $\bound{\tilt} = 2$
respectively;
see Section~\ref{sec:set-sampling}
for another example.

\subsection{GFlexics: Generic pattern sampler}

The first variant relies 
on \cpfim{} \cite{Guns2011}, 
a constraint programming-based 
mining system.
A wide range of constraints
supported by \cpfim{}
are automatically 
supported by the sampler and
can be freely combined
with various quality measures.

\begin{figure*}[t]
\newcommand{\geGap}{\hskip 15pt}
\newcommand{\geCaption}[3]{%
	\multicolumn{#1}{@{} c @{\geGap}}{\pbox[t]{#2}{\centering #3}}%
}
\newcommand{\geCaptionLast}[3]{%
	\multicolumn{#1}{@{} c @{}}{\pbox[t]{#2}{\centering #3}}%
}

\setlength{\tabcolsep}{2pt}%
\begin{tabular}{l @{\geGap} ccccc|c @{\geGap} cccccc|c @{\geGap} ccccc|c @{\geGap} cccccc|c}
\multicolumn{20}{c}{}                                                                                                                                                 &               & $\downarrow$ &   &   &   & \cencol{$\downarrow$}{1} &        \\

${\scriptstyle x_1 \otimes x_5 = 1}$                         & 1 & 0 & 0 & 0 & 1 & 1  &               & 1 &     0  &     0  & 0 & 1 &     1  & 1 & 0 & 0 & 0 & 1 & 1  & $\rightarrow$ & \bf{0}       & 0 & 0 & 0 & \bf{0}                   & \bf{1} \\
${\scriptstyle x_2 \otimes x_3 \otimes x_4 \otimes x_5 = 0}$ & 0 & 1 & 1 & 1 & 1 & 0  & $\rightarrow$ & 0 & \bf{1} &     0  & 0 & 0 & \bf{0} & 0 & 0 & 0 & 1 & 1 & 1  &               & \bf{0}       & 0 & 0 & 1 & \bf{0}                   & \bf{0} \\
${\scriptstyle x_1 \otimes x_2 \otimes x_3 \otimes x_5 = 0}$ & 1 & 1 & 1 & 0 & 1 & 0  & $\rightarrow$ & 0 &     0  & \bf{1} & 0 & 0 & \bf{1} & 0 & 0 & 0 & 0 & 0 & 0  &               &     0        & 0 & 0 & 0 &     0                    &     0  \\
${\scriptstyle x_2 \otimes x_4 \otimes x_5 = 1}$             & 0 & 1 & 0 & 1 & 1 & 1  &               & 0 &     0  &     0  & 1 & 1 &     1  & 0 & 0 & 0 & 0 & 0 & 0  & 		      &     0        & 0 & 0 & 0 &     0                    &     0  \\

\\[-6.5pt]

\geCaption{1}{2.2cm}{\centering 1) Random XOR constraints} & 
	\geCaption{6}{1.5cm}{\centering 2) Initial constraint matrix} & 
	\geCaption{7}{2cm}{\centering 3) Echelonized matrix: assign- ments $x_2=0$ and $x_3=1$ are derived} &
	\geCaption{6}{1.7cm}{\centering 4) Updated matrix (rows 2 and 4 are swapped)} & 
	\geCaptionLast{7}{2.2cm}{\centering 5) If $x_1$ and $x_5$ are set to 1 (e.g., by search), the system is unsatisfiable} \\

\end{tabular}
 \caption{Propagating %
XOR constraints using
Gaussian elimination 
in $\mathbb{F}_2$.}
\label{fig:ge-example}
\end{figure*}

In order to turn \cpfim{}
into a suitable bounded oracle,
we need to extend it with
an efficient propagator
for XOR constraints. 
This propagator is based 
on the process of 
\emph{Gaussian elimination}
\cite{Gomes2007}, 
a classical algorithm 
for solving systems 
of linear equations.
Each XOR constraint can be
viewed as a linear equality 
over the field $\mathbb{F}_2$ of 
two elements, $0$ and $1$, 
and all coefficients 
form a binary matrix
(\autoref{fig:ge-example}.2).
At each step, the matrix is
updated with the latest
variable assignments and
transformed to
\emph{row echelon form}, where
all ones are on or above
the main diagonal and
all non-zero rows are 
above any rows of all zeroes
(\autoref{fig:ge-example}.3).
During echelonization, 
two situations enable propagation.
If a row becomes empty while
its right hand side 
is equal to $1$, 
the system is unsatisfiable and
the current search branch terminates
(\autoref{fig:ge-example}.5).
If a row contains only 
one free variable, 
it is assigned
the right hand side of the row
(\autoref{fig:ge-example}.3).

Gaussian elimination 
in $\mathbb{F}_2$ 
can be performed
very efficiently, because 
no division is necessary
(all coefficients are $1$), and 
subtraction and addition are 
equivalent operations.
For a system of
$k$ XOR constraints
over $n$ variables,
the total time complexity of
Gaussian elimination is
$\bigoh{k^2n}$.

\subsection{EFlexics: Efficient pattern sampler}

Generic constraint solvers
currently cannot compete with
the efficiency and scalability of
specialized mining algorithms.
In order to develop
a less flexible, yet more efficient
version of our sampler,
we extend the well-known 
\eclat{} algorithm
to handle XOR constraints.
Thus, \algNameSpec{} is tailored 
for frequent itemset sampling and
uses \eclat{}XOR 
(Algorithm \ref{alg:eclat-xor})
as an oracle.

Algorithm \ref{alg:eclat-xor}
shows the pseudocode 
of the extended \eclat{}.
The algorithm relies on
the \emph{vertical} data
representation, i.e.,
for each candidate item,
it stores a set of indices of
transactions (TIDs),
in which this item occurs
(Line~\ref{line:tids}).
\eclat{} starts with
determining frequent items and
ordering them,
by frequency ascending.
It explores the search space
in a depth-first manner,
where each branch
corresponds to (ordered)
itemsets that share a prefix.

\begin{algorithm}
\caption{\eclat{}
augmented with 
XOR constraint propagation
(Lines \ref{line:xor-start}-\ref{line:xor-end})}
\newcommand{\MarkXorPropagation}{%
	\hspace*{0.5cm}%
    \smash{\raisebox{-1.75cm}{$%
		\left\{
			\begin{array}{@{}c@{}}
				\\{}\\{}\\{}\\{}\\{}\\{}\\{}\\{}\\{}\\{}
			\end{array}
		\right.%
	$}}
}
\begin{algorithmic}[1]
\Require Dataset $\DB$ over items $\Items$, min.freq $\minsup$,
	XOR matrix $M$
\Ensure Item order $\succ_\Items$ by frequency ascending
\Function{EclatXOR}{$\DB$, $\minsup$, $M$}
	\Statex \Comment{\emph{Mine all frequent patterns that satisfy XOR constraints encoded by $M$}}
	\State Frequent items $FI = \emptyset$
			\label{line:singleton-start}
	\For{item $\Item \in \Items$}
		\State $TID_\Item = \set{\text{transaction index }%
				\Transaction \in \Transactions\mid%
				\DB_{\Transaction\Item}=1}$
			\label{line:tids}
		\If{$\abs{TID_\Item} \geq \minsup$} \Comment Item is frequent
			\State $FI \overset{Add}{\leftarrow} \tuple{\Item, TID_\Item}$
		\EndIf
	\EndFor
	\State \Call{Sort}{$FI$, $\succ_\Items$}
		\label{line:singleton-end}
	\For{$\Item \in FI$}
		\State Candidate suffixes $CS = \set{\Item^\prime \in FI \setminus \Item \mid \Item^\prime >_\Items \Item}$
			\label{line:suffixes-singleton}
		\State \Call{EqClass}{$\set{\Item}$, $CS$, $M$}
	\EndFor
\EndFunction
\Function{EqClass}{Prefix $P$, cand.suffixes $CS \neq \emptyset$, $M$}%
	\Statex \Comment{\emph{Mine all patterns that start with $P$}}
	\If{\Call{CheckConstraints}{$P$, $M$}}
		\State \Return $P$ \Comment Return prefix, if it satisfies XORs
	\EndIf
	\For{candidate suffix $s \in CS$}
		\State $P^{\prime} = P \cup s$; frequent suffixes $FS =$
		\Statex $\set{f \in CS \setminus s \mid%
					f >_\Items s \land \abs{f.TID \cap s.TID} \geq \minsup}$
			\label{line:suffixes-eq}
		\Statex \MarkXorPropagation{} \Comment{\emph{Propagate XOR constraints}}
		\State $U_1 = \set{s}$, $U_0 = CS \setminus FS$ \Comment{Variable updates}
			\label{line:xor-start}
		\State $M^{\prime} =$ \Call{UpdateAndEchelonize}{$M$, $U_1$, $U_0$}
		\State $\tuple{A_1, A_0} = $ \Call{Propagate}{$M^\prime$}
			\label{line:xor-propagate}
			\Comment{Item variables}
		\Statex \Comment{that were assigned value $1$ or $0$ by propagation}
		\State $FS^\prime = FS \setminus \left(A_1 \cup A_0\right)$
			\label{line:xor-update}
		\If{$A_1 \neq \emptyset$} \Comment If prefix was extended,
		\Statex \Comment{update TIDs and check support}
		\State $P^{\prime} \leftarrow P^{\prime} \cup A_1$,
					$\Delta_{TID} = \bigcap\limits_{f \in A_1}{f.TID}$
			\label{line:xor-sup-check}
		\State $FS^{\prime} \leftarrow \set{f^\prime \in FS^\prime :%
					\abs{f^{\prime}.TID \cap \Delta_{TID}} \geq \minsup}$
			\label{line:xor-end}
		\EndIf
		\Statex %
		\If{$\abs{P^\prime.TID} \geq \minsup \land FS^{\prime} \neq \emptyset$}
			\State \Call{EqClass}{$P^{\prime}$, $FS^{\prime\prime}$, $M^{\prime}$}
		\EndIf
	\EndFor
\EndFunction
\end{algorithmic} \label{alg:eclat-xor}
\end{algorithm}

The core operation is
referred to as \emph{processing
an equivalence class} 
of itemsets (\algname{EqClass}).
For each prefix,
\eclat{} maintains a set of
candidate suffixes, i.e.,
items that follow
the last item of the prefix
in the item order and
are frequent.
The frequency of
a candidate suffix,
given the prefix,
is computed by
intersecting its TID with
the TID of the prefix
(Lines~\ref{line:suffixes-singleton},
\ref{line:suffixes-eq}, and
\ref{line:xor-end}).

We extend \eclat{} 
with XOR constraint handling
(Lines~\ref{line:xor-start}-\ref{line:xor-end}).
Variable updates stem from
\eclat{} extending the prefix
and removing infrequent suffixes
(Line~\ref{line:xor-start}).
XOR propagation can result in 
extending the prefix or
removing candidate suffixes
as well (Line~\ref{line:xor-update}).
Furthermore, if the prefix 
has been extended,
TIDs of candidate suffixes
need to be updated,
with some of them possibly
becoming infrequent,
leading to further propagation
(Lines~\ref{line:xor-update}-\ref{line:xor-end}).
If the prefix becomes infrequent,
the search branch terminates.

Fixed variable-order search,
like \eclat{}, is
an advantageous case for
Gaussian elimination
\cite{Soos2010}:
non-zero elements
are restricted to
the right region of
the matrix, hence
Gaussian elimination only
needs to consider 
a contiguous, progressively shrinking
subset of columns.
Total memory overhead of \eclat{}XOR
compared to plain \eclat{} is
$\bigoh{d \times \abs{\mathcal{F}} \times N_{XOR} +%
pivot \times \tilt}$,
where $d$ denotes
maximal search depth,
$\abs{\mathcal{F}}$
the number of frequent singletons
(columns of a matrix), and 
$N_{XOR}$ the number of
XOR constraints
(rows of a matrix).
The first term refers to 
a set of XOR matrices in
unexplored search branches,
whereas the second term 
refers to storing 
itemsets in a cell
(Line \ref{line:gen-cell} in
Algorithm \ref{alg:wg}
in Appendix \ref{sec:appendix:wg}).
\section{Pattern set sampling} \label{sec:set-sampling}

We highlight the flexibility
of \algName{} by introducing and
tackling the novel task of
\emph{sampling sets of patterns}.
For the purposes of sampling,
a set of patterns is essentially
treated as a composite pattern.
Typically, constituent patterns
are required to be
different from each other.
The quality (and hence,
the sampling probability)
of a pattern set depends on 
collective properties of
constituent patterns.
These characteristics,
coupled with the immense 
size of the pattern set 
search space,
make sampling even 
more challenging.

To develop a sampler, 
we extend \algNameGen{} with
the CSP-formulation of
the $k$-pattern set mining task
\cite{Guns2013},
which in turn builds upon
the formulation of
the itemset mining task
described in Section \ref{sec:prelim}.
Recall that a CSP is defined by
a set of variables and
constraints over these variables.
Each constituent pattern is
modeled with distinct
item and transaction variables,
i.e., $\ItemVar_{\Item{}k}$ and
$\TransVar_{\Transaction{}k}$ 
for the $k$th pattern $p_k$.
Note that this increases
the length of XOR constraints,
which poses an additional challenge
from the sampling perspective.

Any single-pattern constraint
can be enforced for
a constituent pattern, e.g.,
$\fnc{minfreq}{\minsup}$,
$closed$, or $\fnc{minlen}{\minlen}$.
A common pattern set-specific 
constraint is $no\ overlap$,
which enforces that 
neither the itemsets (1),
nor the sets of transactions 
that they cover (2) overlap:

\smallskip

{\centering \begin{tabular}{c@{\hskip 20pt}c}
	(1)\ $\forall \Item \in \Items               \ \sum{\ItemVar_{\Item{}k}} \leq 1$ &%
	(2)\ $\forall \Transaction \in \Transactions \ \sum{\TransVar_{\Transaction{}k}} \leq 1$ 
\end{tabular}\par}

\smallskip

\noindent Furthermore,
there is typically
a symmetry-breaking
constraint that requires that
the set of transaction indices of
$p_i$ lexicographically precedes
those of $\set{p_j\ \mid\ \ j > i}$.
This approach allows modeling
a wide range of
pattern set sampling tasks,
e.g., sampling $k$-term DNFs,
conceptual clusterings,
redescriptions, and others.
In this paper, we use 
the problem of
\emph{tiling datasets} 
\cite{Geerts2004}
as an example.

The main aim of tiling is
to cover a large number of 
$1$s in a binary $0/1$ dataset
with a given number of patterns.
Thus, \emph{a tiling} is 
essentially a set of itemsets
that together describe
as many item occurrences
as possible.
Without loss of generality, 
we describe the task of sampling
non-overlapping $2$-tilings ($k=2$).
Let $p_1$ and $p_2$ denote
the constituent patterns of
a $2$-tiling.
The quality of a tiling is 
equal to its \emph{area}, 
i.e., the number of $1$s
that it covers:
\begin{equation*}
\fnc{area}{\set{p_1, p_2}} = \left(%
	\freq{p_1} \times \abs{p_1} +%
	\freq{p_2} \times \abs{p_2}%
\right)
\end{equation*}

The scaling constant for 
$area$ is $C=\sum{\DB_{ti}}$, i.e.,
the total number of $1$s 
in the dataset.
The tilt bound is %
$\bound{\tilt} = \left. \sum{\DB_{ti}} \middle/%
	\left(2 \times \left(\abs{\DB} \times \minsup\right) \times \minlen\right)%
\right.$, where
the denominator is
the smallest possible area
of a $2$-tiling 
given the constraints.
\section{Experiments}%
\label{sec:experiments}

The experimental evaluation
focuses on
accuracy, scalability, 
and flexibility of 
the proposed sampler.
The research questions 
are as follows:
\begin{description}

\item[Q1] \emph{How close is
	the empirical sampling distribution
	to the target distribution?}
\item[Q2] \emph{How does \algName{}
	compare to the specialized alternatives?}
\item[Q3] \emph{Does \algName{}
	scale to large datasets?}
\item[Q4] \emph{How flexible is \algName{}, i.e.,
	can it be used for new pattern sampling tasks?}

\end{description}

The implementations of 
\algNameGen{} and \algNameSpec{}%
\footnote{%
	Available at
	\url{https://bitbucket.org/wxd/flexics}.
}
are based on \algname{cp4im}%
\footnote{\url{https://dtai.cs.kuleuven.be/CP4IM}}
and a custom implementation of \eclat{}
respectively.
Both are augmented with 
a propagator for
a system of XOR constraints
based on the implementation of
Gaussian elimination in the
\algname{m4ri} library%
\footnote{\url{https://bitbucket.org/malb/m4ri/}}
\cite{M4RI}.
All experiments 
were run on a Linux machine
with an Intel Xeon CPU@3.2GHz
and 32Gb of RAM.

\paragraph{Q1: Sampling accuracy}
We study the sampling accuracy
of \algNameGen{} in settings with
tight constraints,
which yield a relatively
low number of solutions.
This allows us to compute
the exact statistical distance 
between the empirical sampling distribution 
and  the target distribution.
We investigate settings with
various quality measures and
constraint sets as well as
the effect of 
the tolerance parameter $\kappa$.

We select several datasets from 
the CP4IM repository\footnote{%
	Source: \url{https://dtai.cs.kuleuven.be/CP4IM/datasets/}}
in the following way.
For each dataset, we construct
two constraint sets (see
\autoref{table:settings}).
We choose a value of $\minsup$
such that there are approximately
$60\,000$ frequent patterns.
Given $\minsup$, we choose
a value of $\minlen \geq 2$
such that there are at least
$15\,000$ closed patterns that 
satisfy the $minlen$ constraint.
In order to obtain sufficiently
challenging sampling tasks,
we omit the datasets where
the latter condition does not hold
(i.e., there are too few
closed ``long'' patterns).
Combining two constraint sets 
with three quality measures
yields six experimental settings 
per dataset.
\autoref{table:accuracy-multiple}
shows dataset statistics
and parameter values.
For each $\kappa \in \set{0.1, 0.5, 0.9}$,
we request $900\,000$ samples.

\begin{table}
\caption{Combinations of
two constraint sets and
three quality measures yield 
six experimental settings per 
dataset for sampling accuracy experiments;
see Section~\ref{sec:prelim}
for definitions.}

\smallskip

\begin{minipage}[t]{0.6\textwidth}
\centering%
\begin{tabular}{lcc}\toprule
    & \multirow{2}{*}{Constraints $\Constraints$} & Itemsets    \\
    &                                             & per dataset \\
\midrule
F   & $\fnc{min\mathbf{F}req}{\minsup}$ & $\sim60\,000$ \\[0.1cm]
\multirow{2}{*}{FCL} %
    & $\fnc{min\mathbf{F}req}{\minsup} \land$                & \multirow{2}{*}{$\geq15\,000$}\\
	& $\mathbf{C}losed \land \fnc{min\mathbf{L}en}{\minlen}$ & \\
\bottomrule
\end{tabular}

\hfill

\end{minipage}
\begin{minipage}[t]{0.4\textwidth}
\centering%
\begin{tabular}{lc}\toprule
Quality         & Tilt                  \\ 
measure $\Qual$ & bound $\bound{\tilt}$ \\
\midrule
$uniform$ ($\Qual \equiv 1$)%
          & $1$ \\[0.05cm]
$purity$  & $2$ \\[0.05cm]
$freq$    & $\minsup^{-1}$ \\
\bottomrule
\end{tabular}
\end{minipage}
\label{table:settings}
\end{table}

Let $T$ denote the set of
all itemsets that satisfy
the constraints,
$E$ denote the 
multiset of all samples,
and $\multiplicity{S}$
its multiplicity function.
For a given quality measure $\Qual$,
target and empirical probabilities of
sampling an itemset $\Pattern$ are
respectively defined as
$\Prob[T]{\Pattern} = \fnc{\Qual}{\Pattern} /%
	\sum\limits_{\Pattern^\prime \in T}{\fnc{\Qual}{\Pattern^\prime}}$
and
$\Prob[E]{\Pattern} = \fnc{\multiplicity{E}}{\Pattern} / \abs{E}$. 
We use \emph{Jensen-Shannon (JS) divergence}
to quantify the statistical distance
between $\ProbSymbol_T$ and $\ProbSymbol_E$.
Let $\fnc{D_{KL}}{\ProbSymbol_1 \| \ProbSymbol_2}$
denote the well-known
\emph{Kullback-Leibler divergence}
between distributions
$\ProbSymbol_1$ and $\ProbSymbol_2$.
JS-divergence $D_{JS}$ is defined as follows:
\begin{align*}
\fnc{D_{JS}}{\ProbSymbol_T \| \ProbSymbol_E} = & \ 0.5 \times \left(%
	\fnc{D_{KL}}{\ProbSymbol_T \| \ProbSymbol_M} +
	\fnc{D_{KL}}{\ProbSymbol_E \| \ProbSymbol_M} 
\right) \\
\text{where }\ProbSymbol_M = &\ 0.5 \times \left(\ProbSymbol_T + \ProbSymbol_E\right)
\end{align*}
\noindent JS-divergence ranges from $0$ to $1$ and,
unlike KL-divergence,
does not require that
$\Prob[T]{\Pattern} > 0 \Rightarrow \Prob[E]{\Pattern} > 0$,
i.e., that each solution is
sampled at least once, which
does not always hold in
sampling experiments.
We compare $D_{JS}$ 
attained with our sampler 
with that of the ideal sampler, 
which materializes all itemsets
satisfying the constraints,
computes their qualities, and
uses these to sample directly
from the target distribution.

\parhead{A characteristic
experiment in detail}
Our experiments show that 
results are consistent
across various datasets.
Therefore, we first study 
the results on the \texttt{vote} 
dataset in detail.
\autoref{table:accuracy-vote}
shows that the theoretical
error tolerance parameter 
$\kappa$ has no
considerable effect on
practical performance of
the algorithm, except for
runtime, which we evaluate
in subsequent experiments.
One possible explanation is
the high quality of the output 
of the estimation phase, 
which thus alleviates
theoretical risks that
have to be accounted for
in the general case
(see below for
a numerical characterization).
Hence, in the following experiments
we use $\kappa=0.9$
unless noted otherwise.

JS-divergences for 
different quality measures
and constraint sets are
impressively low,
equivalent to the highest
possible sampling accuracy
attainable with the ideal sampler.
\autoref{fig:vote-ppplot}
illustrates this for 
$\fnc{minfreq}{0.09} \land closed \land \fnc{minlen}{7}$,
$\Qual = freq$,
and $\kappa=0.9$
($D_{JS}=0.004$):
the sampling frequency of
an average itemset
is close to the target probability.
For at least $90\%$ of patterns,
the sampling error does not
exceed a factor of $2$.

\begin{figure}[t]
\centering \texttt{vote}, 
           $\fnc{minfreq}{0.09} \land closed \land \fnc{minlen}{7}$, 
           $\Qual = freq$ \\
           $\kappa = 0.9$/$\fnc{\varepsilon}{\kappa} = 100.38$;
           $D_{JS}=0.004$
\begingroup
  \makeatletter
  \providecommand\color[2][]{%
    \GenericError{(gnuplot) \space\space\space\@spaces}{%
      Package color not loaded in conjunction with
      terminal option `colourtext'%
    }{See the gnuplot documentation for explanation.%
    }{Either use 'blacktext' in gnuplot or load the package
      color.sty in LaTeX.}%
    \renewcommand\color[2][]{}%
  }%
  \providecommand\includegraphics[2][]{%
    \GenericError{(gnuplot) \space\space\space\@spaces}{%
      Package graphicx or graphics not loaded%
    }{See the gnuplot documentation for explanation.%
    }{The gnuplot epslatex terminal needs graphicx.sty or graphics.sty.}%
    \renewcommand\includegraphics[2][]{}%
  }%
  \providecommand\rotatebox[2]{#2}%
  \@ifundefined{ifGPcolor}{%
    \newif\ifGPcolor
    \GPcolortrue
  }{}%
  \@ifundefined{ifGPblacktext}{%
    \newif\ifGPblacktext
    \GPblacktextfalse
  }{}%
  \let\gplgaddtomacro\g@addto@macro
  \gdef\gplbacktext{}%
  \gdef\gplfronttext{}%
  \makeatother
  \ifGPblacktext
    \def\colorrgb#1{}%
    \def\colorgray#1{}%
  \else
    \ifGPcolor
      \def\colorrgb#1{\color[rgb]{#1}}%
      \def\colorgray#1{\color[gray]{#1}}%
      \expandafter\def\csname LTw\endcsname{\color{white}}%
      \expandafter\def\csname LTb\endcsname{\color{black}}%
      \expandafter\def\csname LTa\endcsname{\color{black}}%
      \expandafter\def\csname LT0\endcsname{\color[rgb]{1,0,0}}%
      \expandafter\def\csname LT1\endcsname{\color[rgb]{0,1,0}}%
      \expandafter\def\csname LT2\endcsname{\color[rgb]{0,0,1}}%
      \expandafter\def\csname LT3\endcsname{\color[rgb]{1,0,1}}%
      \expandafter\def\csname LT4\endcsname{\color[rgb]{0,1,1}}%
      \expandafter\def\csname LT5\endcsname{\color[rgb]{1,1,0}}%
      \expandafter\def\csname LT6\endcsname{\color[rgb]{0,0,0}}%
      \expandafter\def\csname LT7\endcsname{\color[rgb]{1,0.3,0}}%
      \expandafter\def\csname LT8\endcsname{\color[rgb]{0.5,0.5,0.5}}%
    \else
      \def\colorrgb#1{\color{black}}%
      \def\colorgray#1{\color[gray]{#1}}%
      \expandafter\def\csname LTw\endcsname{\color{white}}%
      \expandafter\def\csname LTb\endcsname{\color{black}}%
      \expandafter\def\csname LTa\endcsname{\color{black}}%
      \expandafter\def\csname LT0\endcsname{\color{black}}%
      \expandafter\def\csname LT1\endcsname{\color{black}}%
      \expandafter\def\csname LT2\endcsname{\color{black}}%
      \expandafter\def\csname LT3\endcsname{\color{black}}%
      \expandafter\def\csname LT4\endcsname{\color{black}}%
      \expandafter\def\csname LT5\endcsname{\color{black}}%
      \expandafter\def\csname LT6\endcsname{\color{black}}%
      \expandafter\def\csname LT7\endcsname{\color{black}}%
      \expandafter\def\csname LT8\endcsname{\color{black}}%
    \fi
  \fi
    \setlength{\unitlength}{0.0500bp}%
    \ifx\gptboxheight\undefined%
      \newlength{\gptboxheight}%
      \newlength{\gptboxwidth}%
      \newsavebox{\gptboxtext}%
    \fi%
    \setlength{\fboxrule}{0.5pt}%
    \setlength{\fboxsep}{1pt}%
\begin{picture}(4680.00,3276.00)%
    \gplgaddtomacro\gplbacktext{%
      \csname LTb\endcsname%
      \put(924,686){\makebox(0,0)[r]{\strut{}$8.00 \cdot 10^{-6}$}}%
      \put(924,2809){\makebox(0,0)[r]{\strut{}$9.20 \cdot 10^{-5}$}}%
      \put(1806,264){\makebox(0,0){\strut{}$2.32 \cdot 10^{-5}$}}%
      \put(3851,264){\makebox(0,0){\strut{}$8.66 \cdot 10^{-5}$}}%
      \colorrgb{0.00,0.00,1.00}%
      \put(1605,888){\makebox(0,0)[l]{\strut{}\tiny 5\%}}%
      \put(1443,1071){\makebox(0,0)[l]{\strut{}Avg}}%
      \put(1540,1293){\makebox(0,0)[l]{\strut{}\tiny  95\%}}%
      \colorrgb{1.00,0.00,0.00}%
      \put(3638,3011){\makebox(0,0)[l]{\strut{}Target}}%
      \put(1798,2960){\makebox(0,0)[l]{\strut{}Target$\times 2$}}%
      \put(3638,1874){\makebox(0,0)[l]{\strut{}Target$\times 0.5$}}%
    }%
    \gplgaddtomacro\gplfronttext{%
      \csname LTb\endcsname%
      \put(418,1747){\makebox(0,0){\strut{}\shortstack{Empirical\\probability}}}%
      \put(2801,154){\makebox(0,0){\strut{}\shortstack{Target\\probability}}}%
    }%
    \gplgaddtomacro\gplbacktext{%
    }%
    \gplgaddtomacro\gplfronttext{%
      \csname LTb\endcsname%
      \put(3110,827){\makebox(0,0){\strut{}\shortstack{\tiny Bounds\\\tiny (log)}}}%
    }%
    \gplbacktext
    \put(0,0){\includegraphics{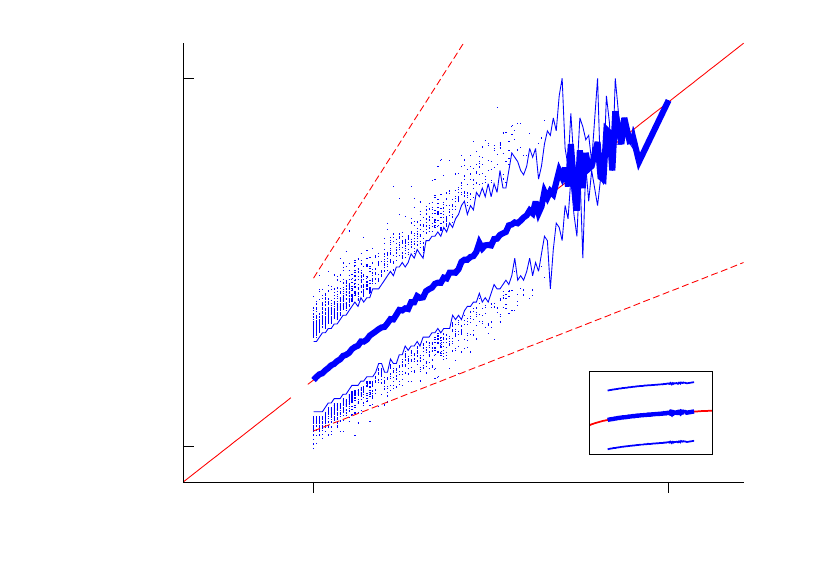}}%
    \gplfronttext
  \end{picture}%
\endgroup
 \caption{Empirical 
sampling frequencies
of itemsets that share
the same target probability,
i.e., have the same quality.
On average, frequencies
are close to the
target probabilities. 
90\% of frequencies are 
well within a factor 2
from the target, which is 
considerably lower than
the theoretical factor 
of 100.38.
(The dots show the tails of 
the empirical probability distribution
for a given target probability.
The lower right box shows
theoretical bounds and 
empirical frequencies
on the logscale).}
\label{fig:vote-ppplot}
\end{figure}

\autoref{table:accuracy-multiple}
shows that similar conclusions
hold for several other datasets.
Over all experimental settings,
the error of the estimation of
the total weight of all solutions,
which is used to derive
the number of XOR constraints
for the sampling phase,
never exceeds $10\%$, whereas
the bounds assume the error
of $45$ to $80\%$.
This helps explain why
practical errors are
considerably lower than
theoretical bounds.

In line with theoretical expectations
(see Section~\ref{sec:algorithm}),
the \texttt{splice} dataset
proves the most challenging
due to the large number of items
(variables in XOR constraints).
As a result, \algNameGen{} does not
generate the requested number of samples
within the 24-hour timeout.
We study the runtime in 
the following experiment.

\begin{table}[t]
\caption{%
Sampling accuracy of
\algName{}
(here \algNameGen{})
is consistently high across 
quality measures,
constraint sets
($\fnc{min\mathbf{F}req}{0.09}$
vs. $\fnc{min\mathbf{F}req}{0.09} \land
	\mathbf{C}losed \land 
	\fnc{min\mathbf{L}en}{7}$), and
error tolerance $\kappa$.
JS-divergence
is impressively low,
equivalent to that of 
the ideal sampler.
}
\centering%
\begin{tabular}{ccccccc}\toprule
 & \multicolumn{6}{c}{\texttt{vote} dataset, JS-divergence from target}  \\
\cmidrule{2-7}
 & \multicolumn{2}{c}{Uniform ($\bound{\tilt}=1$)} & \multicolumn{2}{c}{Purity ($\bound{\tilt}=2$)} & \multicolumn{2}{c}{Frequency ($\bound{\tilt}=11$)} \\
$\kappa$        & F       & FCL     & F       & FCL     & F       & FCL \\
\midrule
$0.9$           & $0.013$ & $0.004$ & $0.013$ & $0.004$ & $0.013$ & $0.004$ \\
$0.5$           & $0.013$ & $0.004$ & $0.013$ & $0.004$ & $0.013$ & $0.004$ \\
$0.1$           & $0.013$ & $0.004$ & $0.013$ & $0.004$ & $0.013$ & $0.004$ \\
\midrule
Ideal sampler   & $0.013$ & $0.004$ & $0.013$ & $0.004$ & $0.013$ & $0.004$ \\
\bottomrule
\end{tabular}
\label{table:accuracy-vote}
\end{table}

\begin{table}
\caption{%
Dataset statistics and
parameter values and
results of sampling
accuracy experiments.
Even with high
error tolerance 
$\kappa=0.9$,
JS-divergence 
of \algName{} 
(here \algNameGen{}) 
is consistently low
across datasets,
quality measures, 
and constraint sets.
(On the \texttt{splice} dataset,
\algNameGen{} generates 
less than $900\,000$ samples
before the timeout;
see also Table~\ref{table:runtime-comparison}.)}
\centering%
\setlength{\tabcolsep}{3.5pt}
\begin{tabular}{lcccc@{\hskip 2pt}cc@{\hskip 15pt}c@{\hskip 4pt}cc@{\hskip 4pt}cc@{\hskip 4pt}c}\toprule
\multicolumn{7}{c}{} & \multicolumn{6}{c}{JS-divergence, $\kappa=0.9$} \\
\cmidrule{8-13}
                     &             &                &         & \cencol{}{2}          &           & \multicolumn{2}{c}{Uniform} & \multicolumn{2}{c}{Purity} & \multicolumn{2}{c}{Frequency} \\
                     & $\abs{\DB}$ & $\abs{\Items}$ & Density & \cencol{$\minsup$}{2} & $\minlen$ & F & FCL & F & FCL & F & FCL \\
\midrule
\texttt{german}      & $1000$ & $112$ & $34\%$ & $0.35$ & ($349$)  & $2$ & $0.012$ & $0.003$ & $0.013$ & $0.003$ & $0.013$ & $0.003$ \\
\texttt{heart}       & $296$  & $95$  & $47\%$ & $0.43$ & ($127$)  & $2$ & $0.012$ & $0.003$ & $0.012$ & $0.003$ & $0.012$ & $0.003$ \\
\texttt{hepatitis}   & $137$  & $68$  & $50\%$ & $0.39$ & ($53$)   & $5$ & $0.013$ & $0.004$ & $0.014$ & $0.004$ & $0.013$ & $0.004$ \\
\texttt{kr-vs-kp}    & $3196$ & $74$  & $49\%$ & $0.69$ & ($2190$) & $6$ & $0.013$ & $0.005$ & $0.013$ & $0.005$ & $0.013$ & $0.005$ \\
\texttt{primary}     & $336$  & $31$  & $48\%$ & $0.09$ & ($30$)   & $7$ & $0.013$ & $0.004$ & $0.013$ & $0.004$ & $0.013$ & $0.004$ \\
\texttt{splice}      & $3190$ & $287$ & $21\%$ & $0.04$ & ($122$)  & $3$ & $-$     & $-$     & $-$     & $-$     & $-$     & $-$     \\
\texttt{vote}        & $435$  & $48$  & $33\%$ & $0.09$ & ($40$)   & $7$ & $0.013$ & $0.004$ & $0.013$ & $0.004$ & $0.013$ & $0.004$ \\

\bottomrule
\end{tabular}

 \label{table:accuracy-multiple}
\end{table}

\begin{table}
\caption{%
The accuracy of 
\algName{} 
(here \algNameGen{})
is consistent 
across settings.
In uniform frequent 
itemset sampling,
performance of
\algName{} as well 
as of ACFI is equivalent to 
that of the ideal sampler
(not shown).
In frequency-weighted sampling,
it is comparable to
the exact two-step sampler
(TS $\sim freq$)
with rejection.
However, the latter
suffers from low
acceptance rates, which,
for settings marked with `$-$',
is not improved by
increasing bias
(TS $\sim freq^4$).
On \texttt{splice}, 
neither TS nor \algName{}
generate $900\,000$ samples
before the timeout;
see also Table~\ref{table:runtime-comparison}.}
\newcommand{\Sci}[2]{${\scriptstyle #1 \cdot 10}^{#2}$}

\centering%
\setlength{\tabcolsep}{2.75pt}
\begin{tabular}{lccccc@{\hskip 0.5pt}rc@{\hskip 0.5pt}rccc@{\hskip 0.5pt}rc@{\hskip 0.5pt}r}\toprule
 & \multicolumn{14}{c}{JS-divergence (for TS, acceptance rate)} \\
\midrule
 & \multicolumn{2}{c}{Uniform} & & \multicolumn{11}{c}{Frequency} \\
 & \multicolumn{2}{c}{F}       & & \multicolumn{5}{c}{F} & \hspace{0.5pt} & \multicolumn{5}{c}{FCL} \\
\cmidrule{2-3} \cmidrule{5-9} \cmidrule{11-15}
 & \cencol{GF}{1} & \cencol{ACFI}{1}    & \hspace{0.5pt} & %
   \cencol{GF}{1} & \cencol{TS$\sim$$freq$}{2} & \cencol{TS$\sim$$freq^4$}{2} & & %
   \cencol{GF}{1} & \cencol{TS$\sim$$freq$}{2} & \cencol{TS$\sim$$freq^4$}{2} \\
\midrule
 \texttt{german}      & $0.01$  & $0.01$  & & $0.01$  & $-$     & (\Sci{9}{-8})  & $-$     & ($0.02$) & & $0.00$  & $-$    & (\Sci{5}{-8}) & $-$     & ($0.06$)      \\
 \texttt{heart}       & $0.01$  & $0.01$  & & $0.01$  & $-$     & (\Sci{4}{-10}) & $-$     & ($0$)    & & $0.00$  & $-$    & ($0$)         & $-$     & (\Sci{3}{-3}) \\
 \texttt{hepatitis}   & $0.01$  & $0.01$  & & $0.01$  & $-$     & (\Sci{2}{-6})  & $-$     & ($0.01$) & & $0.00$  & $-$    & (\Sci{1}{-6}) & $-$     & ($0.01$)      \\
 \texttt{kr-vs-kp}    & $0.01$  & $0.01$  & & $0.01$  & $-$     & (\Sci{7}{-7})  & $-$     & ($0.01$) & & $0.01$  & $-$    & (\Sci{4}{-7}) & $-$     & (\Sci{4}{-3}) \\
 \texttt{primary}     & $0.01$  & $0.01$  & & $0.01$  & $0.01$  & ($0.30$)       & $0.40$  & ($0.99$) & & $0.01$  & $0.01$ & ($0.13$)      & $0.27$  & ($0.10$)      \\
 \texttt{splice}      & $0.01$  & $-$     & & $-$     & $-$     & ($0$)          & $-$     & ($0$)    & & $-$     & $-$    & ($0$)         & $-$     & ($0$)         \\
 \texttt{vote}        & $0.01$  & $0.01$  & & $0.01$  & $0.01$  & ($0.13$)       & $0.23$  & ($0.94$) & & $0.00$  & $0.01$ & ($0.05$)      & $0.14$  & ($0.22$)      \\

\bottomrule
\end{tabular} \label{table:accuracy-comparison}
\end{table}

\begin{table}
\caption{Runtime
in milliseconds
required to sample
a frequent itemset,
including pre-processing, i.e.,
estimation or burn-in,
amortized over $1000$ samples.
Both variants of \algName{}
are suitable for 
anytime exploration, 
although slower than 
the specialized samplers.
The two-step sampler is
the fastest in the task
it is tailored for, but
fails in the settings 
with tighter constraints.
\algNameSpec{} provides
runtime benefits
compared to \algNameGen{}.}

\newcommand{\rtrow}[9]{%
	\texttt{#1} & $#3$ & $#4$ & $#2$ & & $#5$ & $#6$ & $#7$
}

\setlength{\tabcolsep}{3pt}
\centering%
\begin{tabular}{l@{\hskip 10pt}c@{\hskip 10pt}cccccc}\toprule
 & \multicolumn{3}{c}{$\Qual=uniform$, $\Constraints =$ F} & \hspace{10pt} & \multicolumn{3}{c}{$\Qual=freq$, $\Constraints =$ F} \\
\cmidrule{2-4} \cmidrule{6-8}
 & \algNameGen{} & \algNameSpec{} & ACFI & & \algNameGen{} & \algNameSpec{} & TS$\sim$$freq$ \\
\midrule

 \rtrow{german}     {39} {110}   {25}  {133} {34}  {58540}{112}{2.9} \\
 \rtrow{heart}      {24} {60}    {45}  {73}  {44}  {-}    {95} {2.3} \\
 \rtrow{hepatitis}  {11} {23}    {33}  {30}  {45}  {2632} {68} {2.6} \\
 \rtrow{kr-vs-kp}   {6}  {59}    {9}   {59}  {10}  {8731} {74} {1.5} \\
 \rtrow{primary}    {4}  {10}    {10}  {27}  {25}  {0.10} {31} {11.2} \\
 \rtrow{splice}     {580}{170360}{1376}{-}   {1095}{-}    {287}{26.1} \\
 \rtrow{vote}       {8}  {25}    {19}  {46}  {28}  {0.03} {48} {10.9} \\
\bottomrule
\end{tabular} \label{table:runtime-comparison}
\end{table}

\parhead{Q2: Comparison with alternative pattern samplers}
We compare \algName{} to 
ACFI \cite{Boley2009} and 
TS \cite{Boley2012}, 
alternative samplers%
\footnote{%
	The code was provided by
	their respective authors.
	We also obtained 
	the ``unmaintained'' code 
	for the \emph{uniform} LRW sampler
	(personal communication), 
	but were unable to make it run 
	on our machines.
	The code for the FCA sampler
	was not available
	(personal communication).%
}
described in Section~\ref{sec:relwork},
in the settings that
they are tailored for.
ACFI only supports
the setting with a single 
$\fnc{minfreq}{\theta}$ 
constraint and
$\Qual = uniform$.
It is run with
a burn-in of $100\,000$ steps and 
uses a built-in heuristic
to determine the number of steps
between consecutive samples.
TS is evaluated in
the setting with $\Qual = freq$ and
both constraint sets from
the previous experiments.
It samples from two of the distributions
it supports, $freq$ and $freq^4$;
samples that do not
satisfy the constraints
are rejected.
Both samplers are requested
to generate $900\,000$ samples and
are allowed to run up to $24$ hours.
Datasets and parameters
are identical to
the previous experiments.

\autoref{table:accuracy-comparison}
shows the accuracy of the samplers.
The performance of \algName{} is
on par with specialized samplers.
That is, in uniform 
frequent itemset sampling,
the accuracy of both 
\algName{} and ACFI
is equivalent to that of 
the ideal sampler and
can therefore not be improved.
When sampling proportional
to frequency,
it is equivalent
to the accuracy of
the exact two-step sampler
TS $\sim freq$.
However, the latter
does not directly take
constraints into account,
which poses 
considerable problems
on most datasets.
For example, for the
\texttt{heart} dataset,
TS fails to generate
a single accepted sample,
despite generating 2 billion
unconstrained candidates.
This issue is not
solved by increasing
the bias towards
more frequent itemsets
by sampling 
proportional to $freq^4$.
Furthermore, this would
substantially decrease
accuracy, as seen in
\texttt{primary} and
\texttt{vote}.

\autoref{table:runtime-comparison}
shows the runtimes for
frequent itemset sampling
(i.e., only the $minfreq$ constraint).
In most settings,
\algNameSpec{} provides
runtime benefits
over \algNameGen{}.
The \texttt{splice} dataset 
is the most challenging due to 
the large number of items;
it highlights the importance of
an efficient constraint oracle.
Accordingly, the specialized sampler
ACFI is from 6 to 22 milliseconds
faster than a faster variant 
of \algName{} in uniform sampling
(excluding \texttt{splice}).
In frequency-weighted sampling,
\algName{} is considerably faster
in the settings with
tighter constraints,
where the two-step sampler
is slow to generate
accepted samples.
This illustrates 
the overhead as well as
the benefits of
the flexibility of
the proposed approach.
Furthermore, in these settings,
there are at most 
$66\,000$ patterns, which is
too low to suggest
the need for pattern sampling
(recall that the primary goal of
these experiments was to 
evaluate and compare
sampling accuracy)
and does not allow for
the overhead amortization.
We therefore tackle settings 
with a much larger
number of patterns
in the following experiments.

\smallskip

\parhead{Q3: Scalability}
To study scalability of
the proposed sampler,
we compare its runtime costs
with those required
to construct an ideal sampler
with \lcm{}%
\footnote{\url{http://research.nii.ac.jp/~uno/codes.htm}, ver. 3}, 
an efficient 
frequent itemset miner
\cite{Uno2005}.
To this end, we estimate
the costs of completing 
the following scenario:
pre-processing 
(estimation or counting),
followed by sampling
$100$ itemsets in 
two batches of $50$.
We use non-synthetic datasets 
from the FIMI repository\footnote{%
	\url{http://fimi.ua.ac.be/data/}%
}, which have fewer than 
one billion transactions and
select $\minsup$ such that
there are more than one billion
frequent itemsets (see 
Table~\ref{table:runtime-fimi}).

\parhead{A characteristic
experiment in detail}
We use the \texttt{accidents} dataset
($469$ items, $340\,183$ transactions)
and $\minsup=0.009$
(3000 transactions),
which results in
a staggering number of
$5.37$ billion frequent itemsets.
We run \wg{} with values of 
$\kappa \in \set{0.1, 0.5, 0.9}$.
(Note that the estimation phase
is identical for all three cases.)
The baseline sampler is
constructed as follows.
\lcm{} is first run in
counting mode, which only 
returns the total number of itemsets. 
Then, for each batch, $50$ random 
line numbers are drawn,
and the corresponding itemsets 
are printed while \lcm{} is
enumerating the solutions\footnote{%
Storing all itemsets on disk
provides no benefits:
it increases the mining runtime 
to 23 minutes and results in
a file of $215$Gb;
simply counting its lines
with `\texttt{wc -l}'
takes 25 minutes.}.
The latter phase is 
implemented with the standard 
Unix utility `\texttt{awk}`.

\begin{figure}[t]
{\centering \texttt{accidents}, $\fnc{minfreq}{0.009}$, $uniform$\par}

\smallskip

\begin{minipage}[b]{0.5964\textwidth}
\centering %
\begingroup
  \makeatletter
  \providecommand\color[2][]{%
    \GenericError{(gnuplot) \space\space\space\@spaces}{%
      Package color not loaded in conjunction with
      terminal option `colourtext'%
    }{See the gnuplot documentation for explanation.%
    }{Either use 'blacktext' in gnuplot or load the package
      color.sty in LaTeX.}%
    \renewcommand\color[2][]{}%
  }%
  \providecommand\includegraphics[2][]{%
    \GenericError{(gnuplot) \space\space\space\@spaces}{%
      Package graphicx or graphics not loaded%
    }{See the gnuplot documentation for explanation.%
    }{The gnuplot epslatex terminal needs graphicx.sty or graphics.sty.}%
    \renewcommand\includegraphics[2][]{}%
  }%
  \providecommand\rotatebox[2]{#2}%
  \@ifundefined{ifGPcolor}{%
    \newif\ifGPcolor
    \GPcolortrue
  }{}%
  \@ifundefined{ifGPblacktext}{%
    \newif\ifGPblacktext
    \GPblacktextfalse
  }{}%
  \let\gplgaddtomacro\g@addto@macro
  \gdef\gplbacktext{}%
  \gdef\gplfronttext{}%
  \makeatother
  \ifGPblacktext
    \def\colorrgb#1{}%
    \def\colorgray#1{}%
  \else
    \ifGPcolor
      \def\colorrgb#1{\color[rgb]{#1}}%
      \def\colorgray#1{\color[gray]{#1}}%
      \expandafter\def\csname LTw\endcsname{\color{white}}%
      \expandafter\def\csname LTb\endcsname{\color{black}}%
      \expandafter\def\csname LTa\endcsname{\color{black}}%
      \expandafter\def\csname LT0\endcsname{\color[rgb]{1,0,0}}%
      \expandafter\def\csname LT1\endcsname{\color[rgb]{0,1,0}}%
      \expandafter\def\csname LT2\endcsname{\color[rgb]{0,0,1}}%
      \expandafter\def\csname LT3\endcsname{\color[rgb]{1,0,1}}%
      \expandafter\def\csname LT4\endcsname{\color[rgb]{0,1,1}}%
      \expandafter\def\csname LT5\endcsname{\color[rgb]{1,1,0}}%
      \expandafter\def\csname LT6\endcsname{\color[rgb]{0,0,0}}%
      \expandafter\def\csname LT7\endcsname{\color[rgb]{1,0.3,0}}%
      \expandafter\def\csname LT8\endcsname{\color[rgb]{0.5,0.5,0.5}}%
    \else
      \def\colorrgb#1{\color{black}}%
      \def\colorgray#1{\color[gray]{#1}}%
      \expandafter\def\csname LTw\endcsname{\color{white}}%
      \expandafter\def\csname LTb\endcsname{\color{black}}%
      \expandafter\def\csname LTa\endcsname{\color{black}}%
      \expandafter\def\csname LT0\endcsname{\color{black}}%
      \expandafter\def\csname LT1\endcsname{\color{black}}%
      \expandafter\def\csname LT2\endcsname{\color{black}}%
      \expandafter\def\csname LT3\endcsname{\color{black}}%
      \expandafter\def\csname LT4\endcsname{\color{black}}%
      \expandafter\def\csname LT5\endcsname{\color{black}}%
      \expandafter\def\csname LT6\endcsname{\color{black}}%
      \expandafter\def\csname LT7\endcsname{\color{black}}%
      \expandafter\def\csname LT8\endcsname{\color{black}}%
    \fi
  \fi
    \setlength{\unitlength}{0.0500bp}%
    \ifx\gptboxheight\undefined%
      \newlength{\gptboxheight}%
      \newlength{\gptboxwidth}%
      \newsavebox{\gptboxtext}%
    \fi%
    \setlength{\fboxrule}{0.5pt}%
    \setlength{\fboxsep}{1pt}%
\begin{picture}(4294.08,2520.00)%
    \gplgaddtomacro\gplbacktext{%
      \csname LTb\endcsname%
      \put(462,1352){\makebox(0,0)[r]{\strut{}$50$}}%
      \put(462,2265){\makebox(0,0)[r]{\strut{}$100$}}%
      \put(677,300){\makebox(0,0){\scriptsize 2}}%
      \put(942,300){\makebox(0,0){\scriptsize 6}}%
      \put(1864,300){\makebox(0,0){\scriptsize 24}}%
      \put(2232,300){\makebox(0,0){\scriptsize 30}}%
      \put(2497,300){\makebox(0,0){\scriptsize 35}}%
      \put(3318,300){\makebox(0,0){\scriptsize 51}}%
      \colorrgb{0.00,0.00,1.00}%
      \put(916,2392){\makebox(0,0)[l]{\strut{}\shortstack{{\tiny Time}\\{\tiny per sample:}}}}%
      \put(1864,2447){\makebox(0,0){\scriptsize \shortstack{$\kappa=0.9$\\10.3 s}}}%
      \put(2497,2447){\makebox(0,0){\scriptsize \shortstack{$\kappa=0.5$\\18.1 s}}}%
      \put(3318,2447){\makebox(0,0){\scriptsize \shortstack{$\kappa=0.1$\\26.5 s}}}%
      \colorrgb{0.80,0.15,0.16}%
      \put(3841,2447){\makebox(0,0){\scriptsize \shortstack{\lcm{}\\34.8 s}}}%
      \put(889,121){\makebox(0,0)[l]{\strut{}1st LCM batch}}%
      \put(2420,121){\makebox(0,0)[l]{\strut{}2nd LCM batch}}%
    }%
    \gplgaddtomacro\gplfronttext{%
      \csname LTb\endcsname%
      \put(286,2470){\makebox(0,0){\scriptsize Samples}}%
      \put(4027,462){\makebox(0,0){\scriptsize \shortstack{Time,\\min.}}}%
    }%
    \gplbacktext
    \put(0,0){\includegraphics{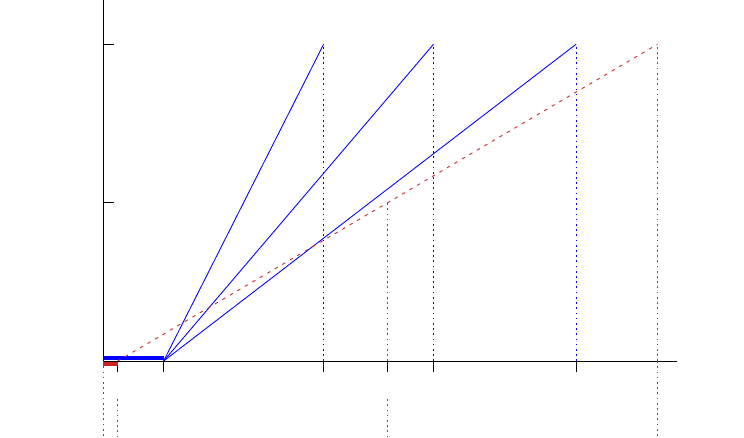}}%
    \gplfronttext
  \end{picture}%
\endgroup
 
a) Sampling runtime comparison
\end{minipage}%
\begin{minipage}[b]{0.3686\textwidth}
\centering %
\begingroup
  \makeatletter
  \providecommand\color[2][]{%
    \GenericError{(gnuplot) \space\space\space\@spaces}{%
      Package color not loaded in conjunction with
      terminal option `colourtext'%
    }{See the gnuplot documentation for explanation.%
    }{Either use 'blacktext' in gnuplot or load the package
      color.sty in LaTeX.}%
    \renewcommand\color[2][]{}%
  }%
  \providecommand\includegraphics[2][]{%
    \GenericError{(gnuplot) \space\space\space\@spaces}{%
      Package graphicx or graphics not loaded%
    }{See the gnuplot documentation for explanation.%
    }{The gnuplot epslatex terminal needs graphicx.sty or graphics.sty.}%
    \renewcommand\includegraphics[2][]{}%
  }%
  \providecommand\rotatebox[2]{#2}%
  \@ifundefined{ifGPcolor}{%
    \newif\ifGPcolor
    \GPcolortrue
  }{}%
  \@ifundefined{ifGPblacktext}{%
    \newif\ifGPblacktext
    \GPblacktextfalse
  }{}%
  \let\gplgaddtomacro\g@addto@macro
  \gdef\gplbacktext{}%
  \gdef\gplfronttext{}%
  \makeatother
  \ifGPblacktext
    \def\colorrgb#1{}%
    \def\colorgray#1{}%
  \else
    \ifGPcolor
      \def\colorrgb#1{\color[rgb]{#1}}%
      \def\colorgray#1{\color[gray]{#1}}%
      \expandafter\def\csname LTw\endcsname{\color{white}}%
      \expandafter\def\csname LTb\endcsname{\color{black}}%
      \expandafter\def\csname LTa\endcsname{\color{black}}%
      \expandafter\def\csname LT0\endcsname{\color[rgb]{1,0,0}}%
      \expandafter\def\csname LT1\endcsname{\color[rgb]{0,1,0}}%
      \expandafter\def\csname LT2\endcsname{\color[rgb]{0,0,1}}%
      \expandafter\def\csname LT3\endcsname{\color[rgb]{1,0,1}}%
      \expandafter\def\csname LT4\endcsname{\color[rgb]{0,1,1}}%
      \expandafter\def\csname LT5\endcsname{\color[rgb]{1,1,0}}%
      \expandafter\def\csname LT6\endcsname{\color[rgb]{0,0,0}}%
      \expandafter\def\csname LT7\endcsname{\color[rgb]{1,0.3,0}}%
      \expandafter\def\csname LT8\endcsname{\color[rgb]{0.5,0.5,0.5}}%
    \else
      \def\colorrgb#1{\color{black}}%
      \def\colorgray#1{\color[gray]{#1}}%
      \expandafter\def\csname LTw\endcsname{\color{white}}%
      \expandafter\def\csname LTb\endcsname{\color{black}}%
      \expandafter\def\csname LTa\endcsname{\color{black}}%
      \expandafter\def\csname LT0\endcsname{\color{black}}%
      \expandafter\def\csname LT1\endcsname{\color{black}}%
      \expandafter\def\csname LT2\endcsname{\color{black}}%
      \expandafter\def\csname LT3\endcsname{\color{black}}%
      \expandafter\def\csname LT4\endcsname{\color{black}}%
      \expandafter\def\csname LT5\endcsname{\color{black}}%
      \expandafter\def\csname LT6\endcsname{\color{black}}%
      \expandafter\def\csname LT7\endcsname{\color{black}}%
      \expandafter\def\csname LT8\endcsname{\color{black}}%
    \fi
  \fi
    \setlength{\unitlength}{0.0500bp}%
    \ifx\gptboxheight\undefined%
      \newlength{\gptboxheight}%
      \newlength{\gptboxwidth}%
      \newsavebox{\gptboxtext}%
    \fi%
    \setlength{\fboxrule}{0.5pt}%
    \setlength{\fboxsep}{1pt}%
\begin{picture}(2653.92,2520.00)%
    \gplgaddtomacro\gplbacktext{%
      \csname LTb\endcsname%
      \put(561,983){\makebox(0,0)[r]{\strut{}$3$}}%
      \put(561,2068){\makebox(0,0)[r]{\strut{}$9$}}%
      \put(797,220){\makebox(0,0){\strut{}$1$}}%
      \put(1214,220){\makebox(0,0){\strut{}$5$}}%
      \put(1630,220){\makebox(0,0){\strut{}$9$}}%
      \put(2046,220){\makebox(0,0){\strut{}$13$}}%
      \put(2463,220){\makebox(0,0){\strut{}$17$}}%
      \colorrgb{0.00,0.00,1.00}%
      \put(1422,2384){\makebox(0,0){\tiny Estimate$\times (1 + \varepsilon_{est})$}}%
      \put(1422,847){\makebox(0,0){\tiny Estimate$/ (1 + \varepsilon_{est})$}}%
      \colorrgb{1.00,0.00,0.00}%
      \put(16,1412){\makebox(0,0)[l]{\strut{}\shortstack{5.37\\{\tiny True count}}}}%
    }%
    \gplgaddtomacro\gplfronttext{%
      \csname LTb\endcsname%
      \put(385,2360){\makebox(0,0){\strut{}\shortstack{\scriptsize Itemsets\\\scriptsize (bln.)}}}%
      \put(2276,550){\makebox(0,0){\scriptsize Iterations}}%
    }%
    \gplbacktext
    \put(0,0){\includegraphics{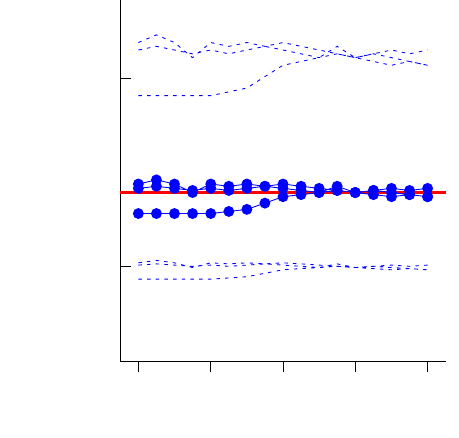}}%
    \gplfronttext
  \end{picture}%
\endgroup
 
b) Estimation accuracy
\end{minipage}

\caption{%
a) \algNameSpec{}
generates two batches of
50 samples faster than 
a sampler derived from \lcm{},
regardless of error tolerance.
b) \algNameSpec{}
with the $uniform$ quality
converges to a high-quality estimate
of the total number of itemsets
in a small number of iterations
(three different random seeds shown).
Practical error of
the estimation phase is 
substantially lower than 
theoretical bounds,
which indirectly signals
high sampling accuracy.}
\label{fig:accidents}
\end{figure}

\autoref{fig:accidents}
illustrates the results.
The counting mode of \lcm{} is
roughly $4.5$ minutes faster
than the estimation phase of
\algNameSpec{}.
Generating samples from 
the output of \lcm{},
on the other hand,
is considerably slower:
it takes approximately 35s
to sample one itemset, whereas
\algNameSpec{} takes 
from 10s to 27s per sample,
depending on error tolerance $\kappa$.
As a result, \algNameSpec{}
samples two batches
faster than \lcm{} regardless
of its parameter values. Moreover,
with $\kappa=0.9$ it samples
all $100$ itemsets even 
before the first batch is 
returned by~\lcm{}.

Thus, the proposed sampler
outperforms a sampler derived from 
an efficient itemset miner,
even though the experimental setup
favors the latter. First,
non-uniform weighted sampling 
would require more advanced
computations with itemsets,
which would increase the costs of
both counting and sampling with \lcm{}.
Second, \algNameSpec{} could also
benefit from the exact count
obtained by \lcm{} and
start sampling after $1.5$ minutes.
Third, the individual itemsets
sampled from the output of
an algorithm based on
deterministic search are not
\emph{exchangeable}.
Figure \ref{fig:accidents-batches}
illustrates this:
due to \lcm{}'s search order,
certain items only occur 
at the beginning of batches, 
while for \algNameSpec{},
the order within a batch is random.

The accuracy of \algName{}
in this scenario can be
evaluated indirectly,
by comparing the estimate of
the total number of itemsets
obtained at the estimation phase
with the actual number.
The error tolerance of
the estimation phase
is $\varepsilon_{est} = 0.8$
(see Appendix~\ref{sec:appendix:wg}
for details).
\autoref{fig:accidents}b
demonstrates that, in practice,
the error is substantially lower
than the theoretical bound.
Furthermore, 
$3$ to $9$ iterations suffice
to obtain an accurate estimate.
Similar to previous experiments,
accurate input
from the estimation phase
alleviates theoretical risks and
is expected to enable
accurate sampling.

\begin{figure}

\newcommand{\itemposimg}[5]{%
\begin{minipage}[b]{#2\textwidth}%
\centering #3 \smallskip

\includegraphics[width=#4,height=2cm,angle=180,origin=c]%
		{accidents-item_position-#1.png}%
		
#5
\end{minipage}%
}

{\centering \texttt{accidents}, $\fnc{minfreq}{0.009}$, $uniform$\par}

\begin{minipage}[b]{0.065\textwidth}
\centering {\scriptsize Items\\in \lcm{}\\search\\order}

\vspace{0.85cm}
\end{minipage}%
\itemposimg{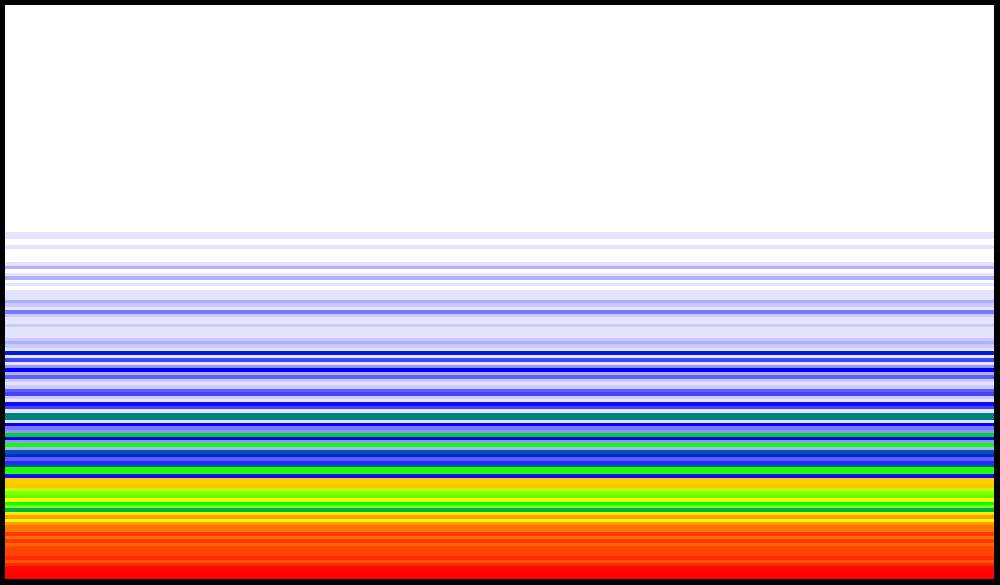}{0.195}{{\scriptsize Expected\\probability (0-0.5)}}{0.3cm}{\ }%
\itemposimg{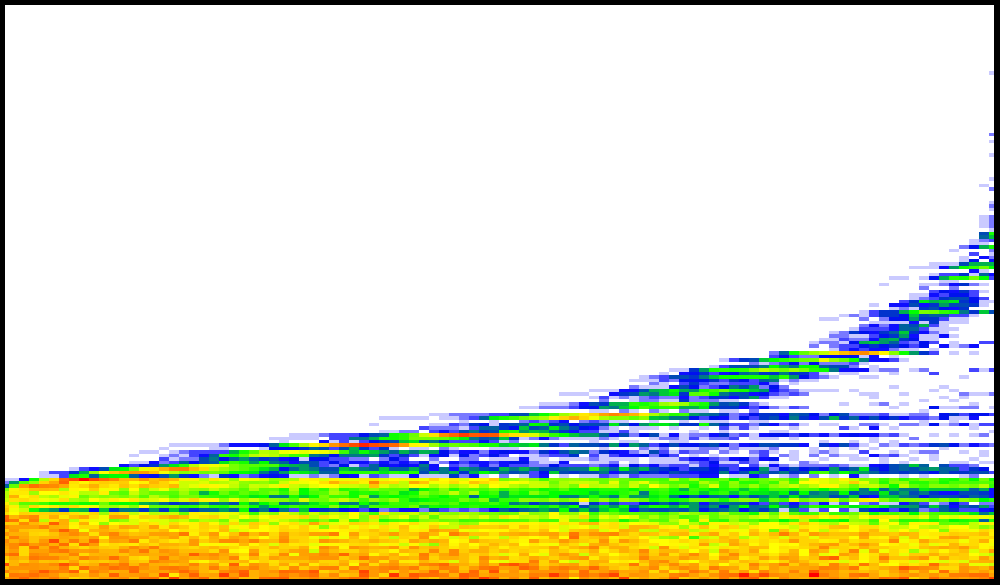}{0.4}{\lcm{}}{3.6cm}{Sample index}%
\itemposimg{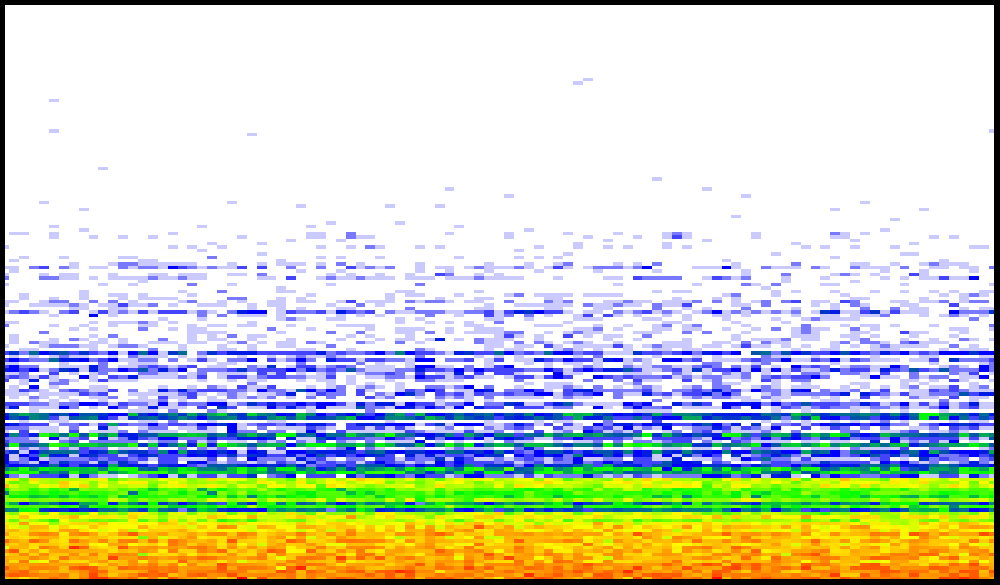}{0.4}{\algNameSpec{}, $\kappa=0.9$}{3.6cm}{Sample index}

\caption{%
The probability of
observing a given item at 
a certain position in
a batch by \algNameSpec{} 
is close to the expected
probability of observing
this item in a random itemset, which 
indicates high sampling accuracy.
The samples by 
the \lcm{}-based sampler
are not exchangeable, 
i.e., certain items
are under- or oversampled
at certain positions in a batch,
depending on their position in
\lcm{}'s search order.}
\label{fig:accidents-batches}
\end{figure}

\begin{table}[t]
\caption{\algNameSpec{}
generates individual samples
considerably faster than \lcm{},
although it is slower in counting. 
The \texttt{kosarak} dataset
poses a significant challenge
to \algNameSpec{} due to
its number of items and sparsity
that complicate the propagation
of XOR constraints.}
\setlength{\tabcolsep}{3pt}
\centering \begin{tabular}{lcccc@{\hskip -1pt}c@{\hskip 10pt}cclcc}\toprule
                     &             &                &                       &           & \mulr{Itemsets,}{2} & \cencol{Counting, min}{2} & & \cencol{Sampling, s}{2}  \\
\cmidrule{7-8} \cmidrule{10-11}
                     & $\abs{\DB}$ & $\abs{\Items}$ & Density               & $\minsup$ & bln.                & \lcm{} & \algNameSpec{}   & & \lcm{}  & \algNameSpec{} \\
\midrule
\texttt{accidents}   & $340183$    & $469$          & $\hphantom{3}7.21\%$  & $0.009$   & $\hphantom{1}5.37$  & $1.55$ & $6.48$           & & $33.77$ & $10.30$  \\
\texttt{connect}     & $67557$     & $130$          &            $33.08\%$  & $0.178$   & $16.88$             & $0.01$ & $0.38$           & & $59.00$ & $0.37$   \\
\texttt{kosarak}     & $990002$    & $41271$        & $\hphantom{3}0.02\%$  & $0.042$   & $10.93$             & $4.87$ & $456.30$         & & $73.04$ & $294.89$ \\
\texttt{pumsb}       & $49046$     & $7117$         & $\hphantom{3}1.04\%$  & $0.145$   & $\hphantom{1}1.11$  & $0.09$ & $1.19$           & & $18.14$ & $0.75$   \\
\bottomrule
\end{tabular} \label{table:runtime-fimi}
\end{table}

Table~\ref{table:runtime-fimi}
summarizes the results.
On three out of four datasets,
\lcm{} is faster in counting
itemsets, but considerably slower
in generating individual samples,
which is even more pronounced
on \texttt{connect} and \texttt{pumsb}
than on \texttt{accidents}.
The results are opposite on
the \texttt{kosarak} dataset,
which is in line with
the theoretical expectations
(see Section~\ref{sec:algorithm}):
the large number of items and
the sparsity of the dataset
sharply increase the costs of
XOR constraint propagation.
As a result, enumeration
with \eclat{} within \algNameSpec{}
becomes considerably slower 
than with \lcm{} (augmenting
\lcm{} to handle XOR constraints
might provide a solution, but 
is challenging from 
an implementation perspective).

\parhead{Q4: Pattern set sampling}
In order to demonstrate
the flexibility of our approach and
the promised benefits of
weighted constrained pattern sampling, i.e.,
1) diversity and quality of results,
2) utility of constraints, and 3) the
potential for anytime exploration,
we here address the problem of
sampling non-overlapping $2$-tilings as
introduced in Section \ref{sec:set-sampling}.
We re-use the implementation of
\algNameGen{} from the
itemset sampling experiments,
only modifying the declarative 
specification of the CSP.
Likewise, we impose 
the FCL constraints
on constituent patterns.

\autoref{table:tiling-runtimes}
shows parameters and runtimes
for sampling $2$-tilings
proportional to $area$.
The time to sample 
a single $2$-tiling is 
suitable for pattern-based
data exploration, where
tilings are inspected by
a human user, as
it exceeds $5$s only on
the \texttt{german} dataset.
For several settings,
the estimation phase 
runtime slightly
exceeds the runtime of 
enumerating all solutions.
However, for the settings
with a large number of 
pattern sets, which are
arguably the primary target 
of pattern samplers,
the opposite is true.
For example, in 
the \texttt{vote} experiment
with $170$ million tilings,
the estimation phase runtime
only amounts to $8\%$ of 
the complete enumeration runtime,
which demonstrates
the benefits of
the proposed approach.

\begin{table}[t]

\newcommand{\tilingrtrow}[8]{%
	\texttt{#1} & $#2$ & $#8$ & $#3$ & \multicolumn{1}{r}{$#4$} & $#5$ & $#6$ & $#7$
}

\caption{Time required 
to sample a 2-tiling
is approximately 4s, 
which is suitable 
for anytime exploration.
Runtime benefits of 
the sampling procedure
are the largest for
the settings with
the largest tiling counts
(\texttt{kr-vs-kp},
\texttt{primary}, and 
\texttt{vote}).}

\setlength{\tabcolsep}{4.8pt}
\centering \begin{tabular}{lccccrrr}\toprule
 &                  &                  &                       &          &                                 & \cencol{Sampling with \algNameGen{}}{2} \\
\cmidrule{7-8}
 & \twor{$\minsup$} & \twor{$\minlen$} & Tilt                  & Tilings, & \cencol{Enumeration,}{1} & \cencol{Estimation,}{1} & \cencol{Per sample,}{1}    \\
 &                  &                  & bound $\bound{\tilt}$ & mln.     & \cencol{min}{1}          & \cencol{min}{1}         & \cencol{s}{1}\\
\midrule
\tilingrtrow{german-credit}{0.22}{25.4}{11.2} {8.2} {12.6}{15.3}{3} \\
\tilingrtrow{heart}        {0.30}{13.3}{2.2}  {1.0} {3.3} {3.9}{5}  \\
\tilingrtrow{hepatitis}    {0.26}{12.4}{7.2}  {1.9} {2.6} {3.6}{5}  \\
\tilingrtrow{kr-vs-kp}     {0.31}{13.1}{20.3} {18.5}{3.5} {5.1}{4}  \\
\tilingrtrow{primary}      {0.03}{50.3}{24.9} {5.5} {4.0} {4.5}{5}  \\
\tilingrtrow{vote}         {0.10}{15.3}{170.1}{37.0}{2.9} {4.4}{5}  \\
\bottomrule
\end{tabular}
\label{table:tiling-runtimes}
\end{table}

\begin{figure}[t]

\newcommand{\tilingimg}[1]{%
	\includegraphics[width=0.2715\textwidth,keepaspectratio,angle=270,origin=c]%
		{vote-tiling-#1-arranged.png}%
}

\begin{minipage}{0.3686\textwidth}
\centering%
\begin{tabular}{cc}
Tiling 1      & Tiling 2      \\
$area=1314$   & $966$         \\
\tilingimg{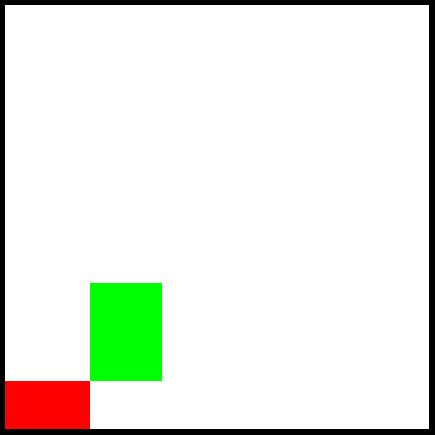} & \tilingimg{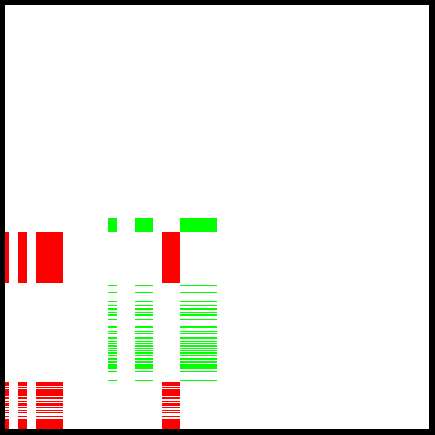} \\
Tiling 3      & Tiling 4      \\
$941$         & $878$         \\
\tilingimg{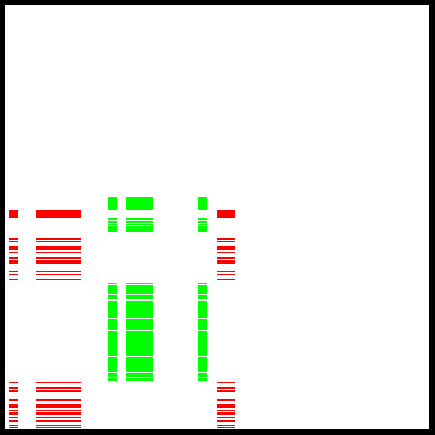} & \tilingimg{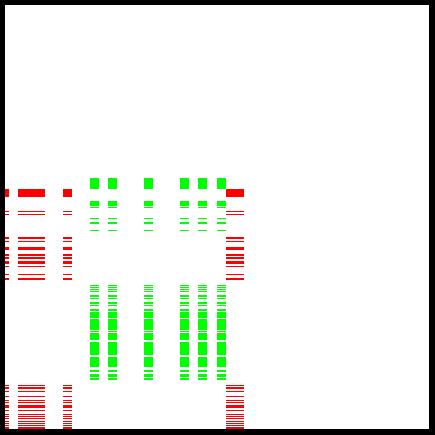} \\
Tiling 5      & Tiling 6      \\
$799$         & $765$         \\
\tilingimg{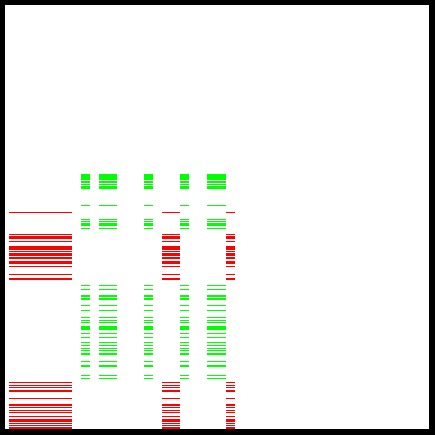} & \tilingimg{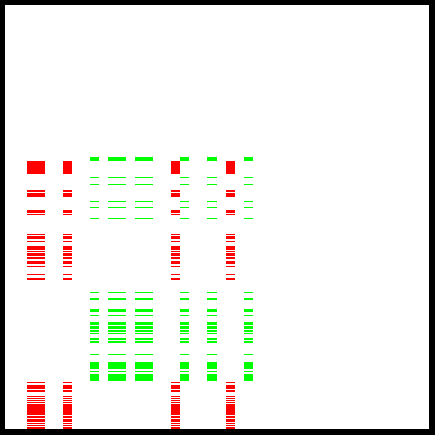} \\
\end{tabular}
\end{minipage}
\begin{minipage}{0.5964\textwidth}
\centering %
\begingroup
  \makeatletter
  \providecommand\color[2][]{%
    \GenericError{(gnuplot) \space\space\space\@spaces}{%
      Package color not loaded in conjunction with
      terminal option `colourtext'%
    }{See the gnuplot documentation for explanation.%
    }{Either use 'blacktext' in gnuplot or load the package
      color.sty in LaTeX.}%
    \renewcommand\color[2][]{}%
  }%
  \providecommand\includegraphics[2][]{%
    \GenericError{(gnuplot) \space\space\space\@spaces}{%
      Package graphicx or graphics not loaded%
    }{See the gnuplot documentation for explanation.%
    }{The gnuplot epslatex terminal needs graphicx.sty or graphics.sty.}%
    \renewcommand\includegraphics[2][]{}%
  }%
  \providecommand\rotatebox[2]{#2}%
  \@ifundefined{ifGPcolor}{%
    \newif\ifGPcolor
    \GPcolortrue
  }{}%
  \@ifundefined{ifGPblacktext}{%
    \newif\ifGPblacktext
    \GPblacktextfalse
  }{}%
  \let\gplgaddtomacro\g@addto@macro
  \gdef\gplbacktext{}%
  \gdef\gplfronttext{}%
  \makeatother
  \ifGPblacktext
    \def\colorrgb#1{}%
    \def\colorgray#1{}%
  \else
    \ifGPcolor
      \def\colorrgb#1{\color[rgb]{#1}}%
      \def\colorgray#1{\color[gray]{#1}}%
      \expandafter\def\csname LTw\endcsname{\color{white}}%
      \expandafter\def\csname LTb\endcsname{\color{black}}%
      \expandafter\def\csname LTa\endcsname{\color{black}}%
      \expandafter\def\csname LT0\endcsname{\color[rgb]{1,0,0}}%
      \expandafter\def\csname LT1\endcsname{\color[rgb]{0,1,0}}%
      \expandafter\def\csname LT2\endcsname{\color[rgb]{0,0,1}}%
      \expandafter\def\csname LT3\endcsname{\color[rgb]{1,0,1}}%
      \expandafter\def\csname LT4\endcsname{\color[rgb]{0,1,1}}%
      \expandafter\def\csname LT5\endcsname{\color[rgb]{1,1,0}}%
      \expandafter\def\csname LT6\endcsname{\color[rgb]{0,0,0}}%
      \expandafter\def\csname LT7\endcsname{\color[rgb]{1,0.3,0}}%
      \expandafter\def\csname LT8\endcsname{\color[rgb]{0.5,0.5,0.5}}%
    \else
      \def\colorrgb#1{\color{black}}%
      \def\colorgray#1{\color[gray]{#1}}%
      \expandafter\def\csname LTw\endcsname{\color{white}}%
      \expandafter\def\csname LTb\endcsname{\color{black}}%
      \expandafter\def\csname LTa\endcsname{\color{black}}%
      \expandafter\def\csname LT0\endcsname{\color{black}}%
      \expandafter\def\csname LT1\endcsname{\color{black}}%
      \expandafter\def\csname LT2\endcsname{\color{black}}%
      \expandafter\def\csname LT3\endcsname{\color{black}}%
      \expandafter\def\csname LT4\endcsname{\color{black}}%
      \expandafter\def\csname LT5\endcsname{\color{black}}%
      \expandafter\def\csname LT6\endcsname{\color{black}}%
      \expandafter\def\csname LT7\endcsname{\color{black}}%
      \expandafter\def\csname LT8\endcsname{\color{black}}%
    \fi
  \fi
    \setlength{\unitlength}{0.0500bp}%
    \ifx\gptboxheight\undefined%
      \newlength{\gptboxheight}%
      \newlength{\gptboxwidth}%
      \newsavebox{\gptboxtext}%
    \fi%
    \setlength{\fboxrule}{0.5pt}%
    \setlength{\fboxsep}{1pt}%
\begin{picture}(4680.00,3276.00)%
    \gplgaddtomacro\gplbacktext{%
      \csname LTb\endcsname%
      \put(528,712){\makebox(0,0)[r]{\strut{}$0.1$}}%
      \put(528,1799){\makebox(0,0)[r]{\strut{}$0.5$}}%
      \put(528,2615){\makebox(0,0)[r]{\strut{}}}%
      \put(764,110){\makebox(0,0){\strut{}\shortstack{440\\Min}}}%
      \put(986,110){\makebox(0,0){\strut{}}}%
      \put(1462,110){\makebox(0,0){\strut{}}}%
      \put(1768,110){\makebox(0,0){\strut{}\shortstack{828\\Median}}}%
      \put(2140,110){\makebox(0,0){\strut{}}}%
      \put(2999,110){\makebox(0,0){\strut{}}}%
      \put(4174,110){\makebox(0,0){\strut{}\shortstack{1758\\Max}}}%
      \put(986,345){\makebox(0,0){\tiny $1\%$}}%
      \put(1462,345){\makebox(0,0){\tiny $25\%$}}%
      \put(2140,345){\makebox(0,0){\tiny $75\%$}}%
      \put(3002,345){\makebox(0,0){\tiny $99\%$}}%
      \colorrgb{0.00,0.00,1.00}%
      \settowidth{\gptboxwidth}{\widthof{\textbf{$n$}}}
	\advance\gptboxwidth by 2\fboxsep
      \savebox{\gptboxtext}{\parbox[c][\totalheight+2\fboxsep]{\gptboxwidth}{\centering{\textbf{$n$}}}}
	\put(3248,2479){\colorbox{white}{\usebox{\gptboxtext}}}
	\settowidth{\gptboxwidth}{\usebox{\gptboxtext}}
	\advance\gptboxwidth by 2\fboxsep
	\put(3248,2479){\framebox[\gptboxwidth][c]{\usebox{\gptboxtext}}}
      \put(3507,2523){\makebox(0,0)[l]{\strut{}Tiling $n$}}%
      \colorrgb{1.00,0.00,0.00}%
      \put(3248,1528){\makebox(0,0)[l]{\strut{}\shortstack{Smooth\\histogram\\of $area$}}}%
    }%
    \gplgaddtomacro\gplfronttext{%
      \csname LTb\endcsname%
      \put(360,2397){\makebox(0,0){\strut{}\shortstack{{\scriptsize Tilings}\\{\scriptsize (mln.)}}}}%
      \put(4137,569){\makebox(0,0){\strut{}$area$}}%
      \put(2471,2945){\makebox(0,0){\strut{}\shortstack{$Area$ distribution of all non-overlapping $2$-tilings\\\texttt{vote}, $\fnc{minfreq}{0.1} \land closed \land \fnc{minlen}{5}$}}}%
      \colorrgb{0.00,0.00,1.00}%
      \settowidth{\gptboxwidth}{\widthof{\textbf{1}}}
	\advance\gptboxwidth by 2\fboxsep
      \savebox{\gptboxtext}{\parbox[c][\totalheight+2\fboxsep]{\gptboxwidth}{\centering{\textbf{1}}}}
	\put(3023,642){\makebox[0.5\width][r]{\colorbox{white}{\usebox{\gptboxtext}}}}
	\settowidth{\gptboxwidth}{\usebox{\gptboxtext}}
	\advance\gptboxwidth by 2\fboxsep
	\put(3023,642){\makebox[0.5\width][r]{\framebox[\gptboxwidth][c]{\usebox{\gptboxtext}}}}
      \colorrgb{0.00,0.00,1.00}%
      \settowidth{\gptboxwidth}{\widthof{\textbf{2}}}
	\advance\gptboxwidth by 2\fboxsep
      \savebox{\gptboxtext}{\parbox[c][\totalheight+2\fboxsep]{\gptboxwidth}{\centering{\textbf{2}}}}
	\put(2123,1833){\makebox[0.5\width][r]{\colorbox{white}{\usebox{\gptboxtext}}}}
	\settowidth{\gptboxwidth}{\usebox{\gptboxtext}}
	\advance\gptboxwidth by 2\fboxsep
	\put(2123,1833){\makebox[0.5\width][r]{\framebox[\gptboxwidth][c]{\usebox{\gptboxtext}}}}
      \colorrgb{0.00,0.00,1.00}%
      \settowidth{\gptboxwidth}{\widthof{\textbf{3}}}
	\advance\gptboxwidth by 2\fboxsep
      \savebox{\gptboxtext}{\parbox[c][\totalheight+2\fboxsep]{\gptboxwidth}{\centering{\textbf{3}}}}
	\put(2058,1014){\makebox[0.5\width][r]{\colorbox{white}{\usebox{\gptboxtext}}}}
	\settowidth{\gptboxwidth}{\usebox{\gptboxtext}}
	\advance\gptboxwidth by 2\fboxsep
	\put(2058,1014){\makebox[0.5\width][r]{\framebox[\gptboxwidth][c]{\usebox{\gptboxtext}}}}
      \colorrgb{0.00,0.00,1.00}%
      \settowidth{\gptboxwidth}{\widthof{\textbf{4}}}
	\advance\gptboxwidth by 2\fboxsep
      \savebox{\gptboxtext}{\parbox[c][\totalheight+2\fboxsep]{\gptboxwidth}{\centering{\textbf{4}}}}
	\put(1895,1292){\makebox[0.5\width][r]{\colorbox{white}{\usebox{\gptboxtext}}}}
	\settowidth{\gptboxwidth}{\usebox{\gptboxtext}}
	\advance\gptboxwidth by 2\fboxsep
	\put(1895,1292){\makebox[0.5\width][r]{\framebox[\gptboxwidth][c]{\usebox{\gptboxtext}}}}
      \colorrgb{0.00,0.00,1.00}%
      \settowidth{\gptboxwidth}{\widthof{\textbf{5}}}
	\advance\gptboxwidth by 2\fboxsep
      \savebox{\gptboxtext}{\parbox[c][\totalheight+2\fboxsep]{\gptboxwidth}{\centering{\textbf{5}}}}
	\put(1691,1170){\makebox[0.5\width][r]{\colorbox{white}{\usebox{\gptboxtext}}}}
	\settowidth{\gptboxwidth}{\usebox{\gptboxtext}}
	\advance\gptboxwidth by 2\fboxsep
	\put(1691,1170){\makebox[0.5\width][r]{\framebox[\gptboxwidth][c]{\usebox{\gptboxtext}}}}
      \colorrgb{0.00,0.00,1.00}%
      \settowidth{\gptboxwidth}{\widthof{\textbf{6}}}
	\advance\gptboxwidth by 2\fboxsep
      \savebox{\gptboxtext}{\parbox[c][\totalheight+2\fboxsep]{\gptboxwidth}{\centering{\textbf{6}}}}
	\put(1603,1713){\makebox[0.5\width][r]{\colorbox{white}{\usebox{\gptboxtext}}}}
	\settowidth{\gptboxwidth}{\usebox{\gptboxtext}}
	\advance\gptboxwidth by 2\fboxsep
	\put(1603,1713){\makebox[0.5\width][r]{\framebox[\gptboxwidth][c]{\usebox{\gptboxtext}}}}
    }%
    \gplbacktext
    \put(0,0){\includegraphics{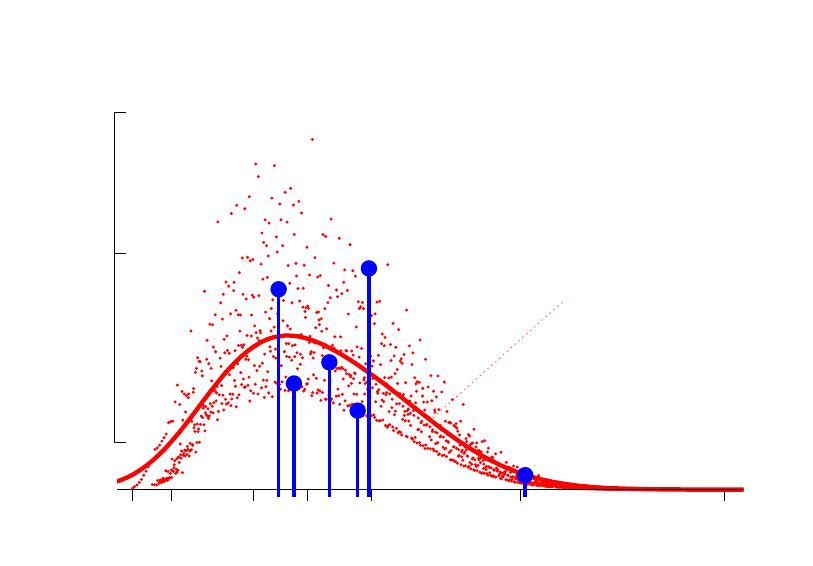}}%
    \gplfronttext
  \end{picture}%
\endgroup
 \end{minipage}
\caption{\textbf{Left}:
Six 2-tilings 
sampled consecutively from 
the \texttt{vote} dataset.
The tilings are diverse, i.e.,
cover different regions in the data,
a property essential for
pattern-based data exploration.
(Note that while 
the sampled tilings are 
fair random draws,
the images are not random:
the tilings were sorted by
$area$ descending, and
items and transactions were
re-arranged so that
the cells covered by 
tilings with larger $area$ are
as close to each other as possible.)
\textbf{Right}: Qualities
($area$) of the samples,
indicated by vertical bars,
tend towards a dense region
between the 25th and
the 75th percentile.}
\label{fig:tilings-vote}
\end{figure}

The left part of
\autoref{fig:tilings-vote} shows
six random $2$-tilings
sampled from the \texttt{vote} dataset.
Constraints ensure
that the individual tiles
comprising each $2$-tiling
do not overlap,
simplifying interpretation.
Moreover, the \emph{set} of 
tilings is diverse, i.e.,
the tilings are 
dissimilar to each other.
They cover different regions 
in the data, revealing 
alternative structural regularities.

The right part of
\autoref{fig:tilings-vote} shows 
the $area$ distribution
of all $2$-tilings that
satisfy the constraints,
obtained by complete enumeration.
Qualities of 5 out of 6 tilings
fall in the dense region
between the $25$th and 
$75$th percentile,
indicating high
sampling accuracy.
This is completely
expected from the 
problem statement.
In practice,
pattern quality measures, 
like $area$, are only
an approximation of
application-specific
pattern interestingness,
thus diversity of results is
a desirable characteristic of
a pattern sampler as long as
the quality of
individual patterns is
sufficiently high.
To sample patterns 
from the right tail (i.e.,
with exceptionally high qualities)
more frequently,
the sampling task
could be changed, e.g.,
either by choosing another 
sampling distribution
or by enforcing 
constraints on $area$.

\section{Discussion}%
\label{sec:discuss}

The experiments demonstrate 
that \algName{} delivers 
the promised benefits:
1) it is flexible in that 
it supports a wide range of
pattern constraints and
sampling distributions 
in itemset mining as well as
the novel pattern set sampling task;
2) it is anytime in that
the time it takes to generate
random patterns is suitable for 
online data exploration, including
the settings with large datasets or
large solution spaces; and
3) by virtue of high sampling accuracy
in all supported settings,
sampled patterns are diverse,
i.e., originate from different regions
in the solution space.
The theoretical guarantees 
ensure that the empirical observations
extend reliably beyond 
the studied settings.
Furthermore, practical accuracy 
is substantially higher 
than theory guarantees.
The results confirm that
pattern mining can benefit
from the latest advances in AI,
particularly in weighted
constrained sampling for SAT.
In this section, we discuss
potential applications,
advantages, and limitations
of the proposed approach.

The primary application of
pattern sampling involves
showing sampled patterns
directly to the user.
In exploratory data analysis,
the mining task is often
ill-defined, i.e.,
the quality measure and
the constraints reflect
the application-specific
pattern interestingness
only approximately \cite{Carvalho2005}.
Owing to its flexibility,
\algName{} allows experimenting
with various task formulations
using the same algorithm.
Pattern sampling allows
obtaining diverse and
representative sets of patterns
in an anytime manner.
These properties are particularly 
important in \emph{interactive 
mining systems}, which aim at 
returning patterns that are 
\emph{subjectively} interesting 
to the current user.
Boley et al. \cite{Boley2013} used 
two-step samplers in such a system, while 
Dzyuba and van Leeuwen \cite{Dzyuba2017}
proposed to learn low-tilt
subjective quality measures
specifically for sampling with 
\algName{}.

Furthermore, the theoretical guarantees
enable applications beyond
displaying the sampled patterns:
\algName{} can be plugged 
into algorithms that use 
patterns as building blocks 
for pattern-based models, 
yielding anytime versions thereof 
with $\left(\varepsilon,\delta\right)$-%
approximation guarantees of their own 
derived from \algName{}' guarantees.
Example approaches include 
community detection
with \eclat{} \cite{Berlingerio2013}
or outlier detection with
two-step sampling \cite{Giacometti2016}.
The authors note that the formulation 
of the mining task has 
a strong influence on the results
in the respective applications.
\algName{} allows the algorithm designer 
to experiment with these choices
and thus to obtain variants of 
these approaches, perhaps with
better application performance.

The flexibility also provides
algorithmic advantages.
In addition to being agnostic of 
the quality measure $\Qual$ and 
the constraint set $\Constraints$,
\algName{} is also agnostic of 
the underlying solution space 
and the oracle, as long as 1) 
solutions can be encoded with
binary variables and 2)
the oracle supports XOR constraints.
Thus, \algName{} provides 
a principled method to convert
a pattern enumeration algorithm
into a sampling algorithm,
which amounts to implementing
the mechanism to handle 
XOR constraints.
This allows re-using 
algorithmic advances in
pattern mining for developing 
pattern samplers, 
which we accomplished with
\cpfim{} and \eclat{}.

Most importantly, \algName{}' 
black-box nature simplifies 
extensions to new pattern languages.
For example, possible extensions 
of \algNameGen{} cover a variety 
of pattern set languages 
in Guns et al. \cite{Guns2013}, 
e.g., conceptual clustering.
\algNameSpec{} can be 
extended to sample other 
binary pattern languages, 
e.g., association rules 
\cite{Agrawal1996} or 
redescriptions \cite{Ramakrishnan2004}.
In contrast, MCMC algorithms, 
like LRW, are based on local 
neighbourhood enumeration, which
is uncommon in traditional
pattern mining techniques, and 
thus require distinctive
design and implementation
principles for novel problems.

On the other hand,
\algName{} only supports
pattern languages that
can be compactly represented
with binary variables, such as
the itemsets and pattern sets
studied in this paper.
This essentially limits it
to propositional discrete 
(binary, categorical, or
discretized numeric) data.
While in principle
structured pattern languages,
e.g., sequences or graphs,
could also be modeled
using this framework, 
the number of variables
would rise sharply,
which would negatively
affect performance.
Devising hashing-based 
sampling algorithms for 
non-binary domains
is an open problem.
In particular, 
sequence mining
can be encoded with 
integer variables
\cite{Kemmar2014};
generalized XOR constraints
\cite{Gomes2007} is one 
possible research direction.
Alternatively, as 
the \algname{m4ri} library 
\cite{M4RI} that 
we base our implementation on
is optimized for \emph{dense}
$\mathbb{F}_2$ matrices,
certain performance issues 
may be addressed with 
Gaussian elimination algorithms
optimized for sparse matrices
\cite{Bouillaguet2016}.

Another limitation concerns
the bounded tilt assumption
regarding sampling distributions:
many common quality measures,
e.g., $\chi^2$, \emph{information gain}
\cite{Nijssen2009}, or
\emph{weighted relative accuracy}
\cite{Lemmerich2013}, have high or 
even effectively infinite tilts
(if $\Qual$ can be 
arbitrarily close to $0$).
Such quality measures
could be tackled with
divide-and-conquer approaches
\cite[Section 6]{Chakraborty2014}
or alternative estimation techniques
\cite{Ermon2013b}.
This requires the capacity 
to efficiently handle 
constraints of the form
$a \leq \fnc{\Qual}{\Pattern} \leq b$,
which is possible for 
a number of quality measures,
including the ones listed above.
\section{Conclusion} \label{sec:conclusion}

We proposed \algName{},
a flexible pattern sampler
with theoretical guarantees
regarding sampling accuracy.
We leveraged the perspective
on pattern mining as 
a constraint satisfaction problem
and developed the first pattern
sampling algorithm that
builds upon the latest advances 
in sampling solutions in SAT.
Experiments show that
\algName{} delivers 
the promised benefits
regarding flexibility, 
efficiency, and sampling accuracy 
in itemset mining as well as in
the novel task of pattern set sampling
and that it is competitive with 
state-of-the-art alternatives.

Directions for future work 
include extensions to 
richer pattern languages and
relaxing assumptions regarding
sampling distributions
(see Section~\ref{sec:discuss}
for a discussion).
Specializing the sampling procedure 
towards typical mining scenarios
may allow for deriving 
tighter theoretical bounds and
improving the practical performance;
examples include specific 
constraint types (e.g., anti-/monotone),
shapes of sampling distributions
(e.g., right-peaked distributions,
similar to Figure~\ref{fig:tilings-vote}),
and iterative mining.
Following the future developments 
in weighted constrained sampling
in AI may provide insights
for improving various 
aspects of \algName{} or
pattern sampling in general.

\parhead{Acknowledgements}
The authors would like 
to thank Guy Van~den~Broeck 
for useful discussions 
and Martin Albrecht 
for  the support with 
the \algname{m4ri} library.
Vladimir Dzyuba is supported
by FWO-Vlaanderen.
 	
	\bibliographystyle{unsrt}

	\appendix
\section{WeightGen}\label{sec:appendix:wg}

In this section, we present
an extended technical 
description of 
the \wg{} algorithm,
which closely follows
Sections 3 and 4 in 
\cite{Chakraborty2014},
whereas the pseudocode
in Algorithm \ref{alg:wg}
is structured similarly to
that of \algname{UniGen2},
a close cousin of \wg{}
\cite{Chakraborty2015}.
Lines \ref{line:cnt-start}-\ref{line:cnt-end}
correspond to the estimation phase and
Lines \ref{line:gen-start}-\ref{line:gen-end}
correspond to the sampling phase.
\algname{SolveBounded}
stands for the bounded
enumeration oracle.

The parameters of
the estimation phase 
are fixed to particular
theoretically motivated values.
$pivot_{est}$ denotes
the maximal weight of
a cell at the estimation phase;
$pivot_{est}=46$ corresponds to
estimation error tolerance
$\varepsilon_{est}=0.8$
(Line \ref{line:param-pivot}).
If the total weight of 
solutions in a given cell
exceeds $pivot_{est}$,
a new random XOR constraint
is added in order to
eliminate a number of solutions.
Repeating the process
for a number of iterations
increases the confidence of
the estimate, e.g.,
$17$ iterations result
in $1-\delta_{est}=0.8$
(Line \ref{line:cnt-start}).
Note that \algname{Estimate}
essentially estimates
the total weight of
\emph{all} solutions,
from which $N_{XOR}$, 
the initial number of
XOR constraints for 
the sampling phase, is derived
(Line \ref{line:gen-start}).

A similar procedure is 
employed at the sampling phase.
It starts with $N_{XOR}$ constraints and
adds at most \emph{three} 
extra constraints.
The user-chosen 
error tolerance parameter $\kappa$
determines the range
$\range{loThresh}{hiThresh}$,
within which the total weight of
a suitable cell should lie
(Line \ref{line:param-range}).
For example, $\kappa=0.9$
corresponds to range $\range{6.7}{49.4}$.
If a suitable cell can be obtained,
a solution is sampled exactly
from all solutions in the cell;
otherwise, no sample is returned.
Requiring the total cell weight
to exceed a particular value
ensures the lower bound on
the sampling accuracy.

The preceding presentation
makes two simplifying assumptions:
(1) all weights lie in $\range{1/\tilt}{1}$;
(2) adding XOR constraints
never results in 
unsatisfiable subproblems
(empty cells).
The former is relaxed by
multiplying pivots by
$\bound{\weight}_{max} = \bound{\weight}_{min} \times \bound{\tilt} < 1$,
where $\bound{\weight}_{min}$ is
the smallest weight observed so far.
The latter is solved by
simply restarting an iteration
with a newly generated set of constraints.
See Chakraborty et al.
\cite{Chakraborty2014}
for the full explanation,
including the precise formulae
to compute all parameters.

\parhead{Implementation details}
Following suggestions of 
Chakraborty et al.
\cite{Chakraborty2015},
we implement \emph{leapfrogging},
a technique that improves
the performance of
the umbrella sampling procedure
and thus benefits both
\algNameGen{} and \algNameSpec{}.
First, after three iterations
of the estimation phase,
we initialize the following iterations
with a number of
XOR constraints that is equal to 
the smallest number returned in 
the previous iterations
(rather than with zero XORs).
Second, in the sampling phase,
we start with one
XOR constraint more than
the number suggested by theory.
If the cell is too small,
we remove one constraint;
if it is too large, we proceed adding 
(at most two) constraints.
Both modifications are based on
the observation that
theoretical parameter values 
address hypothetical corner cases 
that rarely occur in practice.
Finally, we only run
the estimation phase until
the initial number of
XOR constraints, which 
only depends on the median of 
total weight estimates,
converges. For example, 
if the estimation phase is
supposed to run for
17 iterations,
the convergence can
happen as early as
after 9 iterations.

\begin{algorithm}
\caption{\algname{WeightGen} \cite{Chakraborty2014}}
\begin{algorithmic}[1]
\Require Boolean formula $F$, 
	weight $\weight$, 
	tilt bound $\bound{\tilt}$,
	sampling error tolerance parameter $\kappa$
\Ensure $\fnc{\weight}{\cdot} \in \range{1/\bound{\tilt}}{1}$,
	bounded enumeration algorithm \algname{SolveBounded}%
\vspace{1pt}
\For{$17$ iterations} \Comment{Corresponds to $\delta_{est} = 0.2$}%
		\label{line:cnt-start}
	\State $WeightEstimates \overset{Add}{\leftarrow}$
		\Call{Estimate}{$\emptyset$}%
		\label{line:cnt-estimate}
\EndFor
\State $TotalWeight =$ \Call{Median}{$WeightEstimates$}
	\label{line:cnt-end}
\State $N_{XOR} = \bigoh{\log_2 TotalWeight / \left(1 + \kappa^{-1}\right)}$
	\label{line:gen-start}
\State $loThresh \propto \left(1+\kappa\right)/\kappa^2$, $hiThresh \propto \left(1+\kappa\right)^3/\kappa^2$
	\label{line:param-range}
\For{$N_{samples}$ times}
	\State $InitXORs$ = \{\Call{RandomXOR}{}() $\times N_{XOR}$ times\}
	\State \Call{Generate}{$\kappa$, $\range{loThresh}{hiThresh}$,
		$InitXORs$, $3$}
		\label{line:gen-end}
\EndFor
\Statex
\Function{Estimate}{$XORs$}
	\Statex $\triangleright$ Returns an estimate of the total weight of all solutions
	\State $pivot_{est} = 46$ \Comment{Corresponds to $\varepsilon_{est}=0.8$}
		\label{line:param-pivot}
	\State $Sols \leftarrow$ \Call{SolveBounded}{$F$, $XORs$, $pivot_{est}$}
	\State $CellWeight \leftarrow \sum_{s \in Sols}{\fncapp{\weight}{s}}$
	\If{$CellWeight \leq pivot_{est}$} \Comment{Cell of the ``right'' size}
		\State \Return $CellWeight \times 2^{\abs{XORs}}$ 
	\Else \Comment{Shrink cell by adding XOR constraint}
		\State \Call{Estimate}{$XORs\ \cup$ \textsc{RandomXOR}()}%
			\label{line:cnt-restart}
	\EndIf
\EndFunction
\Statex
\Function{Generate}{$\kappa$, $\range{lT}{hT}$, $XORs$, $i$}
	\Statex $\triangleright$ Returns a random solution of $F$
	\State $Sols \leftarrow$ \Call{SolveBounded}{$F$, $XORs$, $hT$}
	\State $CellWeight \leftarrow \sum_{s \in Sols}{\fncapp{\weight}{s}}$%
		\label{line:gen-cell}
	\If{$CellWeight \in \range{lT}{hT}$} \Comment{Cell of the ``right'' size}
		\State \Return \Call{SampleExactly}{$Sols$, $\weight$}
	\ElsIf{$CellWeight > lT \land i > 0$} \Comment{Cell is too large}
		\State \Call{Generate}%
			{$\kappa$, $\range{lT}{hT}$, $XORs \cup \text{\textsc{RandomXOR}()}$, $i-1$}
	\Else  \Comment{Cell is too small}
		\State \Return $\bot$
	\EndIf
\EndFunction
\end{algorithmic} \label{alg:wg}
\end{algorithm}
 \end{document}